\documentclass{article}

\usepackage{arxiv}

\usepackage[utf8]{inputenc} 
\usepackage[T1]{fontenc}    
\usepackage{hyperref}       
\usepackage{url}            
\usepackage{booktabs}       
\usepackage{amsfonts}       
\usepackage{nicefrac}       
\usepackage{microtype}      
\usepackage{lipsum}		
\usepackage{graphicx}
\usepackage{natbib}
\usepackage{doi}
\usepackage{multicol}

\newcommand{\real}{\mathbb{R}}

\renewcommand{\real}{\mathbb{R}}

\newcommand{\ssubset}{\subset\joinrel\subset}

\DeclareSymbolFont{bbold}{U}{bbold}{m}{n}
\DeclareSymbolFontAlphabet{\mathbbold}{bbold}

\newcommand{\vect}[1]{\mathbbold{#1}}
\newcommand{\vectorones}[1][]{\vect{1}_{#1}}
\newcommand{\vectorzeros}[1][]{\vect{0}_{#1}}

\usepackage{amsmath}
\usepackage{amssymb}
\usepackage{amsthm}
\usepackage{xcolor}
\usepackage{tikz}
\usetikzlibrary{matrix, backgrounds, calc}
\usepackage{hyperref}       
\usepackage{subfig}
\usepackage{mathtools, nccmath}
\usepackage{upgreek}
\usepackage{cancel}

\usepackage{graphicx}

\newtheorem{theorem}{Theorem}
\newtheorem{lemma}[theorem]{Lemma}
\newtheorem{corollary}[theorem]{Corollary}
\newtheorem{proposition}[theorem]{Proposition}
\newtheorem{definition}[theorem]{Definition}

\newtheorem{remark}{Remark}
\newtheorem{problem}{Problem}

\definecolor{gnblue4}{RGB}{0,108,212} 
\definecolor{gnblue6}{RGB}{35,156,255}
\definecolor{gnpurple4}{RGB}{72,46,146} 
\hypersetup{colorlinks=true, linkcolor=gnpurple4, breaklinks=true, urlcolor=gnpurple4, citecolor=gnpurple4}

\title{\emph{Safe-by-Design} Learning via Energy-based\\ Neural Networks}

\author{ \href{https://orcid.org/0009-0000-3444-0838}{\includegraphics[scale=0.06]{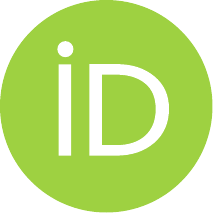}\hspace{1mm}Simone Betteti}\\
	RIAS Lab\\
	The Italian Institute of Artificial Intelligence for Industry\\
	Torino, 10129, ITA \\
	\texttt{\{simone.betteti\}@ai4i.it} \\
	\And
	\href{https://orcid.org/0000-0001-7549-4365}{\includegraphics[scale=0.06]{Fig/orcid.pdf}\hspace{1mm}Morteza Lahijanian} \\
	RECUV \\
	University of Colorado Boulder\\
	Boulder, CO 80303, USA \\
	\texttt{\{morteza.lahijanian\}@colorado.edu} \\
	\AND
	\href{https://orcid.org/0000-0003-1190-6097}{\includegraphics[scale=0.06]{Fig/orcid.pdf}\hspace{1mm}Luca Laurenti\thanks{L. Laurenti is also with the Delft Center for Systems and Control, TU Delft, Delft, 2600 AA, The Netherlands.} } \\
	RIAS Lab\\
	The Italian Institute of Artificial Intelligence for Industry\\
	Torino, 10129, ITA \\
	\texttt{\{luca.laurenti\}@ai4i.it}
}

\renewcommand{\shorttitle}{\emph{Safe-by-design} learning via Energy-Based Neural Ntworks}

\hypersetup{
pdftitle={Safe-by-design Learning via Energy-Based Neural Networks},
pdfauthor={Simone Betteti, Morteza Lahijanian, Luca Laurenti},
pdfkeywords={Energy-based models, Safety, Certifiability, neural ODEs},
}

\begin{document}

\maketitle

\begin{abstract}
Learning neural-network models of dynamical systems with safety guarantees is a fundamental requirement for their deployment in safety-critical settings. Safety is commonly established by proving the invariance of a desired subset in state-space, ensuring that every trajectory initialized in this subset remains confined to it for all time under admissible inputs. Existing frameworks, however, either rely on computationally expensive post-hoc verification or employ safety-enforcing mechanisms without formal correctness guarantees. In this paper, we introduce a novel neural architecture grounded in energy-based modern Hopfield networks to guarantee \emph{safety-by-design} while retaining sufficient expressiveness to model complex nonlinear dynamics. Specifically, we integrate modern Hopfield networks with a port-Hamiltonian neural ODE, enabling by design the construction of barrier functions yielding explicit admissible-input sets and quantitative robustness radii. Across several benchmarks, including an 12-dimensional nanodrone model, our framework achieves state-of-the-art performance while producing certified invariant sets that are more robust to external solicitations than comparable existing approaches.
\end{abstract}

\section{Introduction}

Learning dynamical systems from data is increasingly central to robotics, autonomous systems, and scientific modeling, where neural state-space models can reconstruct complex nonlinear trajectories with remarkable accuracy~\citep{KGE-KIG-LL:21,LC-ST-SG:23}. In safety-critical applications, however, predictive accuracy over observed trajectories is not sufficient. A model that performs well over the finite horizons represented during training may develop unstable or physically undesirable behavior when propagated under new inputs or for longer time intervals. \emph{Safety} is particularly important when learned dynamics are intended for applications such as flight, autonomous operation, or human-robot interaction, where the behavior of the model outside densely sampled training regions is itself operationally relevant~\citep{AR-CR-PC:21,BL-GM-HAW:22,GB-KA:25}. The central question is therefore not only whether a neural model can accurately identify nonlinear dynamics, but whether useful properties of its future evolution can be certified \emph{before} deployment.

Safety in learning-enabled dynamical systems is commonly enforced through external mechanisms on top of an unconstrained predictor, including safety filters, runtime monitors, constrained controllers, and post-hoc verification~\citep{BL-GM-HAW:22,HL-WKP-MM:20}. Certifiability mechanisms are indispensable components of safety-critical systems, but they leave open a complementary question: can certifiability be made an intrinsic property of the learned dynamics themselves? Neural ordinary differential equations (NODEs)~\citep{CRTQ-RY-BJ:18,GS-DM-YJ:19} have motivated several approaches towards safety-certified dynamics, using Lyapunov functions, dissipativity constraints, or algebraic structure to control behavior beyond the training trajectories~\citep{KJZ-MG:19,LN-LP-FM:20,KR-OY:22,KQ-SY-DQ:21,YA-XJ-RM:22}. Port-Hamiltonian neural networks are particularly attractive because they decompose the vector field into energy-preserving interconnection, energy dissipation, and external actuation through explicit input ports~\citep{VDSA:07,MS-PM-BM:20}. Their energy balance provides a natural structural primitive for stability and control without requiring the learned vector field to remain unconstrained.

A central tension nevertheless remains between \emph{certifiability} and \emph{expressivity}. Strong stability guarantees in recent port-Hamiltonian architectures are commonly obtained through restrictive Hamiltonian parameterizations, including convex energies designed around a single equilibrium~\citep{RFJ-KDK-KM:25}. Yet nonconvex energy landscapes are precisely what allow energy-based models to represent multiple stable regimes and complex nonlinear organization. Modern Hopfield networks provide a particularly expressive realization of this principle: advances in energy-based architectures have shown that rich, nonconvex landscapes can coexist with an explicit energy structure~\citep{KD-HJJ:20,HB-CDH-SH:22,HB-LY-PB:23}. For controlled physical dynamics, this suggests a different route to safe learning: rather than simplifying the Hamiltonian until certification becomes tractable, preserve an expressive energy landscape and exploit its geometry directly to determine which external inputs are compatible with safe evolution.

We pursue this approach through a \emph{port-Hamiltonian energy-based model} (pH-EBM), in which a radially unbounded modern Hopfield energy (Fig.~\ref{fig:graphical_abstract}b) parametrizes the Hamiltonian of a controlled neural ODE (Fig.~\ref{fig:graphical_abstract}a). Radial unboundedness confines bounded-energy trajectories to compact sublevel sets, while the interior Hamiltonian remains free to develop nonconvex and multi-well structure. The port-Hamiltonian decomposition provides an explicit energy balance between dissipation and external actuation. We then use the same learned Hamiltonian to define energy barrier functions: for a prescribed sublevel component, the barrier condition identifies the external inputs that cannot inject sufficient energy to drive the learned trajectory across its boundary (Fig.~\ref{fig:graphical_abstract}c). Safety is therefore not introduced as an additional network or post-processing stage, but derived from the energy geometry already governing the neural dynamics. The resulting guarantees concern invariance of prescribed regions of the \emph{learned model state space} under certified classes of inputs; when bounded disturbances or model mismatch are explicitly characterized, the same construction extends to robust certificates.

Our contributions are threefold. \textbf{First}, we introduce a pH-EBM that combines the dissipative structure of port-Hamiltonian NODEs with a coercive, globally nonconvex modern-Hopfield Hamiltonian, separating boundedness at large state norm from the geometry required to represent complex operating regimes. \textbf{Second}, we derive explicit set-valued safety certificates from the learned energy, such as exact state-dependent admissible-input sets, state-independent robustness radii for entire energy regions, and robust extensions under bounded perturbations. We further show how the certified margin is governed by the geometry of the energy shell, dissipation normal to its boundary, and exposure of this direction to the input port, with local curvature providing tractable closed-form bounds. \textbf{Third}, we evaluate the framework across nonlinear system-identification and controlled-mechanics benchmarks spanning classical physical systems, nonconvex multi-attractor dynamics, adversarial out-of-distribution forcing, and high-dimensional flight identification. The experiments show that pH-EBM can offer competitive predictive accuracy, often outperforming state-of-the-art methods, while providing non-trivial safety guarantees by design. 

\begin{figure}
    \centering
    \includegraphics[width=1\linewidth]{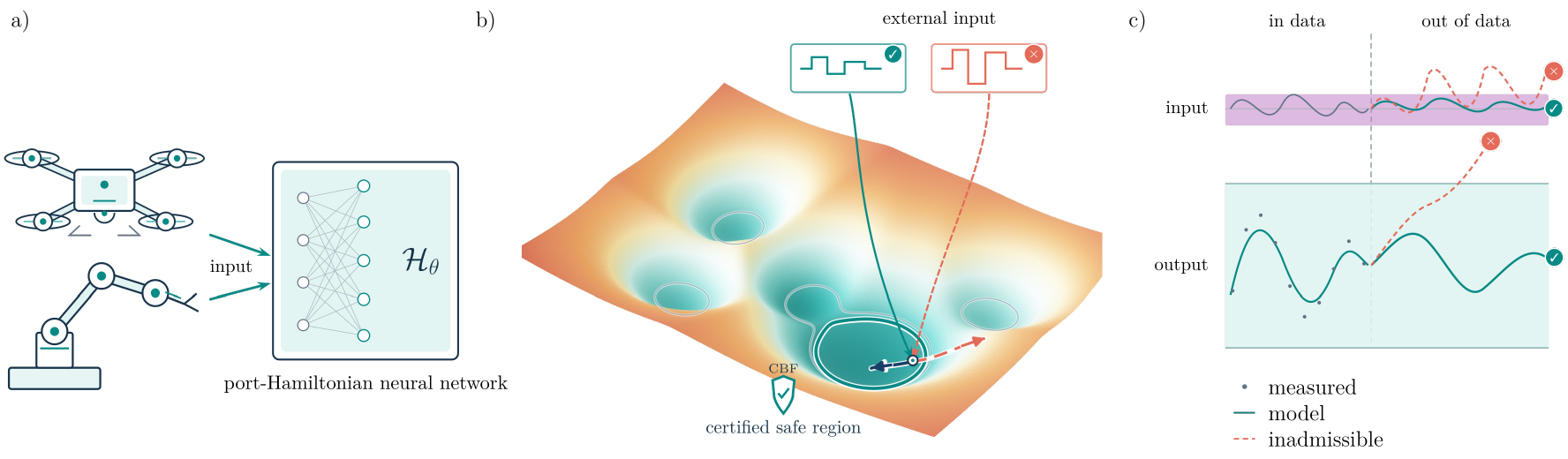}
    \caption{\textbf{\emph{Safe-by-design} Energy-Based Model.} a) Input-output trajectories from physical or simulated systems are used to train the pH-EBM: the external input $u$ drives the neural dynamics, while the measured output $y$ supervises reconstruction. b) The learned Hamiltonian, parameterized by a modern Hopfield energy, organizes the neural ODE into a continuous energy landscape whose low-energy regions encode stable operating regimes. Beyond interpolating between sparsely observed trajectories, this geometry identifies energy sublevel sets that can be certified as invariant for prescribed classes of inputs. c) Inputs lying within the certified admissible set therefore generate trajectories that remain confined to the desired operating region, turning the learned energy landscape into an explicit quantitative safety certificate.}
    \label{fig:graphical_abstract}
\end{figure}

\section{Problem formulation}

We consider the learning of an unknown controlled continuous-time 
dynamical system described by
\begin{align}
\label{eqn:SysDyn}
\dot{x}^\star(t)
&=
f^\star\bigl(x^\star(t),u(t)\bigr), \quad 
y(t)=h^\star\bigl(x^\star(t),u(t)\bigr)+v(t),
\end{align}
where $x^\star(t)\in\real^{n_\star}$ is the generally unobserved physical
state, $u(t)\in\mathcal U\subseteq\real^k$ is a known bounded external input,
$y(t)\in\real^p$ is the measured output, and $v$ denote measurement perturbations. Throughout the paper, we
assume that $f^\star$ satisfies the standard Carathéodory conditions, i.e., locally Lipschitz continuous in $x$ and locally essentially bounded over the admissible values of $u$. These assumptions guarantee the existence and uniqueness of forward-complete Carathéodory solutions~\citep{KHK:02}. Both the physical state $x^\star$ (including its dimensionality) and the vector field $f^\star$ are unknown, and the system can be observed only through sampled inputs $u$ and outputs $y$. The available dataset therefore consists of $M$ experiments, possibly collected from different unknown initial conditions and under different input signals, i.e.,
\begin{equation}
\label{eq:dataset}
\mathcal D
:=
\left\{
\left(
u^{(j)}(t_i),
y^{(j)}(t_i)
\right)_{i=0}^{N_j}
\right\}_{j=1}^{M}.
\end{equation}
Since the physical state $x^\star$ is unobserved, our objective is not necessarily to recover its original realization, but to learn a latent state-space model that reproduces the observed input-output behaviour and generalizes within a neighbourhood sufficiently covered by the data. 
Specifically, we seek a continuous-time neural state-space model of the form
\begin{align}
\label{eq:neurODE}
\dot z(t)
&=
f_\uptheta\bigl(z(t),u(t)\bigr), \quad 
\hat y(t)
=
h_\upphi\bigl(z(t),u(t)\bigr),
\end{align}
where $z(t)\in\real^d$ is the latent model state, $\uptheta$ and
$\upphi$ are trainable parameters, $h_{\upphi}$ is a continuous latent-to-output decoder function, and $\hat{y}(t) \in \real^p$ is the measurement of the latent state $z(t)$.  Note that $\hat{y}$ is in the same space as $y$ of System~(\ref{eqn:SysDyn}). The initial latent state associated
with each experiment is inferred jointly with the model parameters, or
produced by an appropriate encoder. Our objective is not merely to fit the trajectories in $\mathcal D$. Instead, we seek a \emph{safe-by-design} neural architecture. To give this notion a physical interpretation, let
$\mathcal Y_{\mathrm{safe}}\subseteq\real^p$ be a prescribed closed set
of outputs that satisfy the relevant operational safety requirements.
Safety of the learned model is established by jointly constructing
a robustly invariant latent set $\mathcal{Z}_{\mathrm{safe}}$ whose corresponding outputs remain in
$\mathcal Y_{\mathrm{safe}}$.

\begin{definition}[$\uprho$-Robust Invariant Sets]
\label{def:RobustInvariantSet}
Let $\uprho>0$ be an input-intensity threshold and define the associated input set  $u\in\mathcal{U}_\uprho:=\{u\in\mathcal{U}:\: \|u\|\leq \uprho\}$. We say that the compact set of latent states $\mathcal{C} \subset \mathbb{R}^d$ is $\uprho$-Robust invariant if:
\begin{itemize}
	\item  for all $z\in \mathcal{C}$ and $u\in\mathcal{U}_\uprho$  it holds that $h_{\upphi}(z,u) \in \mathcal Y_{\mathrm{safe}}$ and 
    \item for all $z(0)\in\mathcal{C}$ and input signals $u$ such that that $u(t)\in\mathcal{U}_{\uprho}$ for all $t>0$, it holds that $z(t)\in\mathcal{C}$ for all $t>0$.
\end{itemize}
\end{definition}

In this paper, we will establish results for the latent representations $z$, and then exploit the continuity of the decoder map $h_{\upphi}$ to extend safety to compacts in measurement space. Importantly, in our setting the geometry of $\mathcal{Z}_{\mathrm{safe}}$ is generally unknown before training. Since the latent representation $z$, its dynamics $f_\uptheta$, and the output map $h_\Phi$ are learned jointly from data, an invariant set in latent space $\mathcal{Z}_{\mathrm{safe}}$ cannot generally be prescribed a priori. It must instead be learned jointly with the neural state-space realization. This leads to the following problem:

\begin{problem}{\textbf{Learning and certifying neural ODEs.}}\label{prob:robust}
Given a dataset $\mathcal D$, a set $\mathcal{Z}_\mathrm{safe},$ and an input-intensity parameter $\uprho>0$, learn a neural ODE $\dot z(t)=f_\uptheta(z(t),u(t))$ with measured output $\hat y(t)=h_\upphi(z(t),u(t))$ such that for all $\{u(t)\in\mathcal{U}_{\uprho}\}_{t\geq0}$ returns a maximal $\uprho$-Robust Invariant Set $\mathcal{C}\subseteq \mathcal{Z}_{\mathrm{safe}}$.
\end{problem}

Many autonomous systems admit nonempty positively invariant regions, while
systems operating in closed loop with stabilising feedback often admit nonempty robust positively invariant regions under bounded disturbances
\citep{BF:99,BF-MS:15}. The objective of Problem~\ref{prob:robust} is therefore to endow an input-output model learned from data with this structural property: we seek to identify a latent state-space realization while simultaneously guaranteeing that a nontrivial latent safe set is robustly positively invariant. Solving Problem~\ref{prob:robust} is, however, particularly challenging because generic post-hoc methods for computing or certifying invariant sets for neural dynamics, based for instance on gridding, mixed-integer programming, or SMT solving, are computationally demanding and scale poorly with the state dimension and network size~\citep{AA-AD-EA:21,DC-GS-FC:23}. We therefore pursue a complementary route and consider neural architectures that admit a guaranteed $\uprho$-robust invariant set by construction. In particular, in Section~\ref{sec:safedesphNODE} we introduce a
port-Hamiltonian neural architecture whose Hamiltonian sublevel
components provide barrier certificates under explicit dissipation and
input conditions. The barrier certificates ultimately allow us to characterize safe $\uprho$-robustly forward-invariant sets. In Section~\ref{sec:experiments} we instead present the performances of our pH-EBM model in complex dynamics learning tasks, displaying how the model achieves state-of-the-art accuracy and provides explicit safety certificates, using drones' bounded long time rollouts as a paradigmatic proof of concept. 

\section{Safe-by-design port-Hamiltonian neural ODEs}
\label{sec:safedesphNODE}

The solution of Problem~\ref{prob:robust} with a neural ODE prompts the choice of a vector field structure that simultaneously combines two key ingredients: (i) a stability-inducing mechanism to confine trajectories and (ii) an input port. Thereby, following~\citep{RFJ-KDK-KM:25}, we specialize the neural state-space model in~\eqref{eq:neurODE} to the port-Hamiltonian vector field
\begin{equation}
    \dot z
    = f_{\uptheta}(z,u)
    = \left[J_{\uptheta_J}(z)-R_{\uptheta_R}(z)\right]
      \nabla\mathcal H_{\uptheta_H}(z)
      +G_{\uptheta_G}(z)u,
    \label{eq:ph_dynamics}
\end{equation}
where $z\in\real^d$, $u\in\mathcal U\subseteq\real^k$, and $\mathcal H_{\uptheta_H}:\real^d\to\real$ is a continuously differentiable Hamiltonian. Moreover, $J_{\uptheta_J}(z)=-J_{\uptheta_J}(z)^\top\in\real^{d\times d}$ is skew-symmetric,
$0\preceq R_{\uptheta_R}(z)=R_{\uptheta_R}(z)^\top\in\real^{d\times d}$ is symmetric positive semidefinite, and $G_{\uptheta_G}(z)\in\real^{d\times k}$ is the input map. The symmetry constraints on $J_{\uptheta_J}$ and $R_{\uptheta_R}$ are enforced directly through their parameterizations. We assume that $J_{\uptheta_J}$, $R_{\uptheta_R}$, $G_{\uptheta_G}$, and $\nabla\mathcal H_{\uptheta_H}$ are locally Lipschitz in $z$ and that
$u:[0,\infty)\to\mathcal U$ is locally essentially bounded. Consequently,~\eqref{eq:ph_dynamics} admits a unique Carathéodory solution on its maximal interval of existence. Equation~\ref{eq:ph_dynamics} is a standard port-Hamiltonian ODE and satisfies the energy balance~\citep{VDSA:07,WJC:07}
\begin{equation}
\begin{aligned}
\frac{\mathrm d}{\mathrm dt}\mathcal H_{\uptheta_H}(z)
&=
-\nabla\mathcal H_{\uptheta_H}(z)^\top
R_{\uptheta_R}(z)
\nabla\mathcal H_{\uptheta_H}(z)+
\nabla\mathcal H_{\uptheta_H}(z)^\top
G_{\uptheta_G}(z)u.
\end{aligned}
\label{eq:ph_energy_balance}
\end{equation}
Thus, skew-symmetry of $J_{\uptheta_J}$ makes the interconnection term energy preserving, positive semidefiniteness of $R_{\uptheta_R}$ makes the internal dynamics dissipative, and $G_{\uptheta_G}u$ accounts for energy exchanged with the environment. 

Importantly, this balance does not require convexity of the Hamiltonian. Consequently, we now define a non-convex Hamiltonian, which provides the key architectural ingredient for constructing compact confinement regions in the latent space. In particular, we employ the hybrid modern Hopfield Hamiltonian introduced in~\citep{BS-LL:26}:
\begin{equation}
    \mathcal H_{\uptheta_H}(z)
    =
    \frac12\|z\|_2^2
    -\sum_{h=1}^L g_h(z)^\top b_h
    -\sum_{h=2}^{L}
      \mathcal F_h\!\left(g_h(z)\right),
    \label{eq:hybrid_hamiltonian}
\end{equation}
where $L\geq2$ is the number of layers, with widths $N_1=d,N_2,\ldots,N_L$. We set $g_1(z)=z$, $\Psi_1(z)=z$ and recursively define the hidden feature vectors as
\begin{equation}
g_h(z)
=
W_{h(h-1)}
\Psi_{h-1}\!\left(g_{h-1}(z)\right),
\qquad h=2,\ldots,L.
\label{eq:hidden_features}
\end{equation}
Here,  $\mathcal F_h:\real^{N_h}\to\real$ are convex differentiable potentials and $\Psi_h=\nabla\mathcal F_h$ their continuous gradients, $W_{h(h+1)}\in\real^{N_h\times N_{h+1}}$ are interaction matrices between consecutive layers, with tied weights $W_{(h+1)h}=W_{h(h+1)}^\top$, and $b_h\in\real^{N_h}$ are bias vectors entering the Hamiltonian directly. The parameters $\uptheta_H$ collect the interaction matrices, biases, and any trainable parameters of the potentials. Notice that each $g_{h}(z)$ is obtained by composing all preceding layers, allowing increasingly complex features to be encoded throughout the architecture. In addition, each hidden layer $g_H$ is a standard MLP, to which the universal approximation theorem applies~\citep{HK-SM-WH:89}: consequently, the model can in principle approximate any vector field given sufficient data and parameters. The following lemma guarantees that a bounded first hidden activation is sufficient to preserve the quadratic growth of the Hamiltonian at infinity.

\begin{lemma}
\label{lemma:coercivity}
Assume that the first hidden activation $\Psi_2$ has bounded image. Then, $\mathcal H_{\uptheta_H}$ is coercive. That is, there exist constants $c_0,c_1\geq0$ such that
$\mathcal H_{\uptheta_H}(z)
\geq
\frac12\|z\|_2^2-c_1\|z\|_2-c_0.$ 
\end{lemma}
Lemma \ref{lemma:coercivity} follows by the boundedness of $\Psi_2=\nabla\mathcal F_2,$ which implies that
$\mathcal F_2$ grows at most linearly, while all feature vectors $g_{h}(z)$ with $h\geq3$ remain bounded. Hence, the negative terms in~\eqref{eq:hybrid_hamiltonian} grow at most linearly in $\|z\|_2$ and are dominated by its quadratic leading term. Coercivity and continuity of $\mathcal H_{\uptheta_H}$ guarantee that its sublevel sets are compact despite the potentially nonconvex hidden-layer contributions. This is a key property we rely on in the following subsection to construct $\uprho$-Robust Invariant Sets.

\subsection{$\uprho$-Robust Invariant Set Computation}
\label{sec:rob_inv}

To characterize $\uprho$-Robust Invariant Sets for \eqref{eq:ph_dynamics}, we rely on barrier certificates
\citep{AAD-XX-GJW:17}. In particular, we show that an isolated local minimum of $\mathcal H_{\uptheta_H}$ induces a compact energy sublevel component whose robustness to external inputs can be explicitly
quantified. To show that, we let $z_\star$ be an isolated local minimizer of $\mathcal H_{\uptheta_H}$ and define the shifted energy $V_\star(z):=\mathcal H_{\uptheta_H}(z)-\mathcal H_{\uptheta_H}(z_\star)$.
For $\epsilon>0$, consider
\begin{equation}
    \mathcal C_{\epsilon,\star}
    :=
    \operatorname{Comp}_{z_\star}
    \left\{z:h_{\epsilon,\star}(z)\geq0\right\},
    \qquad
    h_{\epsilon,\star}(z)
    :=
    \epsilon-V_\star(z),
    \qquad
    \Gamma_{\epsilon,\star}
    :=
    \partial\mathcal C_{\epsilon,\star},
    \label{eq:component_safe_set}
\end{equation}
where $\operatorname{Comp}_{z_\star}$ denotes the connected component containing $z_\star$. Since $\mathcal H_{\uptheta_H}$ is coercive, $\mathcal C_{\epsilon,\star}$ is compact. We assume that $\mathcal H_{\uptheta_H}(z_\star)=\epsilon$ is a regular value of $\mathcal H_{\uptheta_H}$, so that $\nabla\mathcal H_{\uptheta_H}(z)\neq0$ for every $z\in\Gamma_{\epsilon,\star}$. Given a norm $\|\cdot\|$ on the input space, we denote the corresponding dual norm by $\|\cdot\|_*$ and define
\[
d_H(z)
:=
\nabla\mathcal H_{\uptheta_H}(z)^\top
R_{\uptheta_R}(z)
\nabla\mathcal H_{\uptheta_H}(z)
\geq0,
\qquad
a_H(z)
:=
G_{\uptheta_G}(z)^\top
\nabla\mathcal H_{\uptheta_H}(z)
\in\real^k.
\]
Then, in Theorem \ref{thm:rob_inv} below, we rely on the fact that along the dynamics~\eqref{eq:ph_dynamics} it holds that
\begin{equation}
	\dot h_{\epsilon,\star}(z,u)
=
d_H(z)-a_H(z)^\top u
\geq
d_H(z)-\|a_H(z)\|_*\|u\|,
\end{equation}
to prove that $h_{\epsilon,\star}(z,u)$ is a valid barrier function for System \ref{eq:ph_dynamics}.
\begin{theorem}[\textbf{$\uprho$-robust invariant energy sets}]
\label{thm:rob_inv}
Let $\mathcal H_{\uptheta_H}$ be defined
in~\eqref{eq:hybrid_hamiltonian}, and consider the port-Hamiltonian
dynamics in~\eqref{eq:ph_dynamics}. For every
$z\in\Gamma_{\epsilon,\star}$, define
\begin{equation}
\uprho_{\mathrm{BF}}(z)
:=
\begin{cases}
\tfrac{d_H(z)}{\|a_H(z)\|_*},
& a_H(z)\neq0,\\[1.2ex]
+\infty,
& a_H(z)=0.
\end{cases}
\label{eq:pointwise_norm_bound}
\end{equation}
Then the vector field is inward-pointing or tangent to
$\mathcal C_{\epsilon,\star}$ at $z$ for every input satisfying
$\|u\|\leq\uprho_{\mathrm{BF}}(z)$. Consequently,
$\mathcal C_{\epsilon,\star}$ in~\eqref{eq:component_safe_set} is an
$\uprho_{\epsilon,\star}$-robust invariant set, where
\begin{equation}
	\uprho_{\epsilon,\star}
:=
\inf_{z\in\Gamma_{\epsilon,\star}}
\uprho_{\mathrm{BF}}(z).
\label{eq:uniform_robust_radius}
\end{equation}
\end{theorem}

\begin{proof}
For every $z\in\Gamma_{\epsilon,\star}$, the barrier function condition implies $\dot h_{\epsilon,\star}(z,u)\geq0$ whenever $\|u\|\leq\uprho_{\mathrm{BF}}(z)$. The uniform radius in~\eqref{eq:uniform_robust_radius} therefore makes the vector field
inward-pointing or tangent at every boundary point. Invariance
then follows from the standard barrier-certificate conditions introduced in Appendix~\ref{app:barrier_primer}.
\end{proof}

Theorem~\ref{thm:rob_inv} provides an explicit robustness radius
quantifying the tolerance of each energy component
$\mathcal C_{\epsilon,\star}$ to external inputs. 

\begin{remark}
The construction extends directly to bounded additive perturbations.
In particular, consider $\dot z=f_{\uptheta}(z,u)+w$, $w\in\mathcal W(z)$ where $\mathcal W(z)$ is nonempty and compact, and define its support
function as $\upsigma_{\mathcal W(z)}(q):=\sup_{w\in\mathcal W(z)}q^\top w.$ 
In this case, $\dot h_{\epsilon,\star}\geq d_H(z)-\|a_H(z)\|_*\|u\|-\upsigma_{\mathcal W(z)}\bigl(\nabla\mathcal H_{\uptheta_H}(z)\bigr)$.
Hence, provided that $d_H(z)\geq \upsigma_{\mathcal W(z)}\bigl(\nabla\mathcal H_{\uptheta_H}(z)\bigr)$ for all $z\in\Gamma_{\epsilon,\star}$, the disturbance-robust radius is
\begin{equation}
\uprho_{\epsilon,\star}^{\mathcal W}
:=
\inf_{\substack{z\in\Gamma_{\epsilon,\star}\\a_H(z)\neq0}}
\tfrac{
d_H(z)
-\upsigma_{\mathcal W(z)}
 \bigl(\nabla\mathcal H_{\uptheta_H}(z)\bigr)
}{
\|a_H(z)\|_*
},
\label{eq:disturbance_robust_radius}
\end{equation}
with the quotient interpreted as $+\infty$ when $a_H(z)=0$.
\end{remark} 

\begin{remark}
The norm-ball certificate can also be generalized to arbitrary compact
convex input sets. In particular, a compact convex set $\mathcal U$ is
robustly admissible if ${\small \upsigma_{\mathcal U}(a_H(z))+\upsigma_{\mathcal W(z)}(\nabla\mathcal H_{\uptheta_H}(z))\leq d_H(z)}$, for all $z\in\Gamma_{\epsilon,\star}$. Equivalently, the largest input envelope certified by this condition is 
\begin{equation}
	\overline{\mathcal U}_{\epsilon,\star}^{\mathrm{rob}}:=\mathcal U\cap_{z\in\Gamma_{\epsilon,\star}}\{u\in\real^k:a_H(z)^\top u\leq d_H(z)-\upsigma_{\mathcal W(z)}(\nabla\mathcal H_{\uptheta_H}(z))\}
\end{equation}
This set is convex because it is the intersection of $\mathcal U$ with a family of half-spaces. Thus, the construction accommodates norm
balls, boxes, polytopes, ellipsoids, and asymmetric actuator
constraints.
\end{remark}

\begin{remark} A consequence of Theorem \ref{thm:input_energy} is that if $\mathcal H_{\uptheta_H}$ has finitely many
isolated local minima $\{z_i^\star\}_{i=1}^M$, repeating the
construction for suitable regular energy levels $\epsilon_i$ gives the
robustness atlas
\begin{equation}
	 \mathfrak A
:=\{(
\mathcal C_{\epsilon_i,i},
\overline{\mathcal U}_{\epsilon_i,i}^{\mathrm{rob}},
\uprho_{\epsilon_i,i}
)
\}_{i=1}^M,
\end{equation}
which assigns a certified operating envelope to each learned dynamical
regime.
\end{remark}

\subsection{Computable robustness bounds from local geometry}
\label{sec:computable_robustness}

The exact robust radius in~\eqref{eq:uniform_robust_radius} requires minimizing a state-dependent ratio over the generally nonconvex boundary $\Gamma_{\epsilon,\star}$. Although exact, this characterization
may be computationally demanding and does not explicitly reveal how the
robustness margin depends on the local geometry of the learned model. We
therefore derive a conservative lower bound in terms of the steepness of
the energy shell, dissipation in its normal direction, and exposure of the normal direction to the input port. We then bound the shell-steepness term using a local Polyak-\L{}ojasiewicz condition and derive an input-to-energy estimate under persistent forcing.

For readability, in what follows, we suppress the parameter subscripts in
$\mathcal H_{\uptheta_H}$, $R_{\uptheta_R}$, and $G_{\uptheta_G}$.
We restrict attention to energy levels for which
$\mathcal H(z_\star)+\epsilon$ is a regular value of $\mathcal H$, so
that $\nabla\mathcal H(z)\neq0$ on $\Gamma_{\epsilon,\star}$ and the
outward energy-normal direction $\upeta(z):=\frac{\nabla\mathcal H(z)}{\|\nabla\mathcal H(z)\|_2}$ is well defined. Since $\Gamma_{\epsilon,\star}$ is compact, regularity also implies that the gradient norm is uniformly bounded away from zero on the boundary. We now define the following quantities
\begin{equation}
r_{\epsilon,\star}
:=
\inf_{z\in\Gamma_{\epsilon,\star}}
\upeta(z)^\top R(z)\upeta(z),\quad
\bar g_{\epsilon,\star}^{\perp}
:=
\sup_{z\in\Gamma_{\epsilon,\star}}
\|G(z)^\top\upeta(z)\|_*,\quad
\kappa_{\epsilon,\star}
:=
\inf_{z\in\Gamma_{\epsilon,\star}}
\|\nabla\mathcal H(z)\|_2.
\label{eq:local_geometric_quantities}
\end{equation}
Here, $r_{\epsilon,\star}$ measures normal dissipation,
$\bar g_{\epsilon,\star}^{\perp}$ measures exposure of the energy shell
to the input port, and $\kappa_{\epsilon,\star}$ measures
the minimum steepness of the energy shell. Regularity and compactness
imply $\kappa_{\epsilon,\star}>0$, whereas
$r_{\epsilon,\star}>0$ is an additional dissipativity assumption:
positive semidefiniteness of $R$ alone allows the dissipation to vanish
in the normal direction. We are now ready to state Theorem \ref{thm:geometric_robustness_bound}. 
\begin{theorem}[Geometric lower bound on the robustness radius]
\label{thm:geometric_robustness_bound}
Suppose that $r_{\epsilon,\star}>0$ and
$\bar g_{\epsilon,\star}^{\perp}\neq 0$. Then
\begin{equation}
\uprho_{\epsilon,\star}
\geq
\underline{\uprho}_{\epsilon,\star}^{\mathrm{geo}}
:=
\frac{
r_{\epsilon,\star}\kappa_{\epsilon,\star}
}{
\bar g_{\epsilon,\star}^{\perp}
}.
\label{eq:geometric_robustness_bound}
\end{equation}
Consequently, $\mathcal C_{\epsilon,\star}$ is 
a $\underline{\uprho}_{\epsilon,\star}^{\mathrm{geo}}$-robust invariant set.
\end{theorem}

The proof is provided in Appendix~\ref{app:regular_boundary}. The theorem guarantees that a large overall input gain is not necessarily
detrimental: only its component normal to the energy shell consumes the
available robustness margin.

We next lower-bound $\kappa_{\epsilon,\star}$ using a local
Polyak--\L{}ojasiewicz (PL) condition
\citep{KH-NJ-SM:16,PY-CF-ZL:23}. In our framework, the Hamiltonian is a known scalar, while its gradient is implicitely defined: the local PL condition bounds the growth $\nabla\mathcal{H}$ on a shell $\Gamma_{\epsilon,\star}$with the specified energy budget $\epsilon>0$, and therefore readily provides a local estimate of the $\uprho_{\epsilon,\star}$-robustness radius. Since
$\mathcal H_{\uptheta_H}$ is $C^1$ with locally Lipschitz gradient, the
appropriate local curvature object is its Clarke generalized Hessian,
$\partial_C^2\mathcal H_{\uptheta_H}
:=\partial_C(\nabla\mathcal H_{\uptheta_H})$.

\begin{proposition}[Local Clarke-Hessian certificate for PL]
\label{prop:hessian_pl}
Let $\nabla\mathcal H_{\uptheta_H}(z_\star)=0$, and suppose that
$\mathcal H_{\uptheta_H}$ is $C^1$ with locally Lipschitz gradient in
a neighbourhood of $z_\star$. Define $m_\star:=\min_{Q\in\partial_C^2\mathcal H_{\uptheta_H}(z_\star)}\uplambda_{\min}(Q)$. If $m_\star>0$, then, for every $0<\upmu_\star<m_\star$, there exists a convex neighbourhood $\mathcal N_\updelta$ of $z_\star$ on which
\begin{equation}
\frac12\|\nabla\mathcal H_{\uptheta_H}(z)\|_2^2
\geq
\upmu_\star
\left(
\mathcal H_{\uptheta_H}(z)
-\mathcal H_{\uptheta_H}(z_\star)
\right).
\label{eq:clarke_pl_certificate}
\end{equation}
\end{proposition}

The proposition is strictly local and leaves the global Hamiltonian free
to be nonconvex and multi-well. When $\mathcal H$ is $C^2$ near
$z_\star$, it reduces to the classical positive-definite Hessian test.
The proof and activation-specific formulas are provided in
Appendix~\ref{app:activation_pl}. To translate the PL condition into a robustness bound, the selected component must satisfy $\mathcal C_{\epsilon,\star}\subseteq\mathcal N_\updelta$. We also require uniform dissipation along the energy gradient, i.e., the existence of $r_\star>0$ such that $d_H(z)\geq r_\star\|\nabla\mathcal H(z)\|_2^2$ on $\mathcal N_\updelta$, which is satisfied for neighbourhood of isolated local minima with no flat directions.

\begin{corollary}[PL-based robustness bound]
\label{cor:pl_robustness_bound}
Suppose~\eqref{eq:clarke_pl_certificate} holds for some
$0<\upmu_\star<m_\star$. Then
$\kappa_{\epsilon,\star}\geq\sqrt{2\upmu_\star\epsilon}$ and
$r_{\epsilon,\star}\geq r_\star$. Consequently, if
$\bar g_{\epsilon,\star}^{\perp}>0$, then
\begin{equation}
\uprho_{\epsilon,\star}
\geq
\uprho_{\epsilon,\star}^{\mathrm{PL}}
:=
\frac{
r_\star\sqrt{2\upmu_\star\epsilon}
}{
\bar g_{\epsilon,\star}^{\perp}
}.
\label{eq:regular_boundary_radius_PL}
\end{equation}
\end{corollary}

The same local PL induced geometry gives more than invariance of one prescribed shell, and extends the idea of energy-driven convergence to energy-tube confiment of sollicited trajectories.
\begin{proposition}[Input-to-energy tube]
Let the uniform bounds $\|G(z)\|_2\leq\bar g$, and $\|u(t)\|_2\leq\bar u$ hold in the component $\mathcal{C}_{\epsilon,\star}$  and set $c:=\bar g\bar u$. If $c\leq r_\star\sqrt{2\upmu_\star\epsilon}$ then $\mathcal C_{\epsilon,\star}$ is a $\bar u$-Robust Invariant set. In particular,
\begin{equation}
    \liminf_{t\to\infty}h_{\epsilon,\star}(z(t))
    \geq\epsilon-\frac{c^2}{2r_\star^2\upmu_\star}.
\end{equation}
\end{proposition}
The condition $c\leq r_\star\sqrt{2\upmu_\star\epsilon}$ guarantees that the tube centered at $z_\star$ with radius $c/(2r_\star^2\upmu_\star)$ fits inside $\mathcal C_{\epsilon,\star}$, recovering for $c= r_\star\sqrt{2\upmu_\star\epsilon}$ the robust boundary margin from a dynamic energy estimate.

\begin{table}[!th]
\centering
\scriptsize
\caption{\textbf{Prediction accuracy and certified robustness across nonlinear dynamical systems.} \emph{Comparative performance of the pH-EBM across benchmarks addressing nonlinear system identification, nonconvex dynamics, and higher-dimensional controlled systems.}}
\label{tab:main_experimental_results}
\resizebox{.7\linewidth}{!}{%
\begin{tabular}{@{}lcccc@{}}
\toprule
Benchmark & Published & portHNN-u & pH-EBM\\
Metrics & RMSE $\downarrow\:$ $\uprho_{\epsilon,\star}$ $\uparrow$ & RMSE $\downarrow\:$ $\uprho_{\epsilon,\star}$ $\uparrow$ & RMSE $\downarrow\:$ $\uprho_{\epsilon,\star}$ $\uparrow$\\
\midrule
Silverbox & [$\mathbf{0.293}$,-] & [$47.821$, $2\text{e-}2$] & [$0.422$, $\mathbf{2}\textbf{e-}\mathbf{1}$]\\
CED & [$\mathbf{0.054}$, -] & [$0.217$, $3\text{e-}2$] & [$0.064$, $\mathbf{6}\textbf{e-}\mathbf{1}$]\\
Duffing double-well & [-, -] & [$0.254$, $5\text{e-4}$] & [$\mathbf{0.048}$, $\mathbf{4}\textbf{e-}\mathbf{1}$]\\
$3$-link pendulum & [$0.051$, -] & [$0.276$, $\text{5e-}3$] & [$\mathbf{0.014}$, $\mathbf{2}\textbf{e-}\mathbf{1}$]\\
NanoDrone $0.5\: \mathrm s$ & [$13.712$, -] & [-, -] & [$\mathbf{12.521}$,  $\mathbf{6}\textbf{e-}\mathbf{2}$]\\
NanoDrone $5\: \mathrm s$ & [$1363.701$, -] & [-, -] & [$\mathbf{369.025}$, $\mathbf{6}\textbf{e-}\mathbf{2}$]\\
\bottomrule
\end{tabular}}
\end{table}

\section{Experimental validation}
\label{sec:experiments}

\paragraph{Benchmark suite and evaluation protocol.}
The experimental suite spans established nonlinear system identification (Silverbox~\citep{WT-SJ:013} and CED~\citep{WT-SM:17})\footnote{Dataset specifications and benchmark protocols are available at \url{https://www.nonlinearbenchmark.org}.}, nonconvex controlled mechanics (Duffing~\citep{MFJ-LM-FS:12} and the $n$-link pendulum~\citep{OY-KR:25}), and high-dimensional long-horizon flight dynamics (NanoDrone~\citep{BR-CE-FM:26}); the tasks and main results are summarized in Table~\ref{tab:main_experimental_results}. We compare against benchmark-specific literature references, portHNN-u, and dissipative or black-box NODE baselines where available. The reported pH-EBM models are trained with a Huber observation loss and a mixed AdamW-to-SGD optimization schedule. Prediction is evaluated using each benchmark's native metric, while robustness is assessed using the common energy-shell and normalized-radius protocol detailed in Appendix~\ref{app:num_det}. The pH-EBM framework is available at~\url{https://github.com/sim1bet/energy-safe-dynamics}.

\paragraph{Accuracy and robustness across benchmarks.}
Across the suite, the pH-EBM combines competitive or improved predictive accuracy with substantially larger certified input margins (Table~\ref{tab:main_experimental_results}). On Silverbox and CED, it remains close to the strongest published errors ($0.422$ vs. $0.293$ and $0.064$ vs. $0.054$, respectively) while reducing the portHNN-u error and increasing the corresponding robustness radius by approximately $10\times$ and $20\times$. The advantage is strongest on the nonconvex mechanical systems: on Duffing, RMSE decreases from $0.254$ to $0.048$ while the radius increases from $5\times10^{-4}$ to $4\times10^{-1}$; on the $3$-link pendulum, RMSE decreases from $0.276$ to $0.014$ while the radius increases from $5\times10^{-3}$ to $2\times10^{-1}$. The latter also improves on the published RMSE of $0.051$, and retains lower error under the Brownian and step out-of-distribution forcing tests reported in Appendix~\ref{par:nlink}, while requiring approximately $1/20$ of the training time of the dissipative NODE reference. Finally, on NanoDrone the pH-EBM improves the reported RMSE at both $0.5\,\mathrm{s}$ ($12.521$ vs. $13.712$) and, more markedly, at $5\,\mathrm{s}$ ($369.025$ vs. $1363.701$), while retaining a normalized certificate of $6\times10^{-2}$. These results show that the structural constraints do not impose a systematic accuracy--robustness trade-off: across the tested systems, improved certification is compatible with competitive---and in the nonconvex and long-horizon regimes, substantially improved---prediction.

\begin{figure}[!t]
	\center
	\includegraphics[width=.9\linewidth]{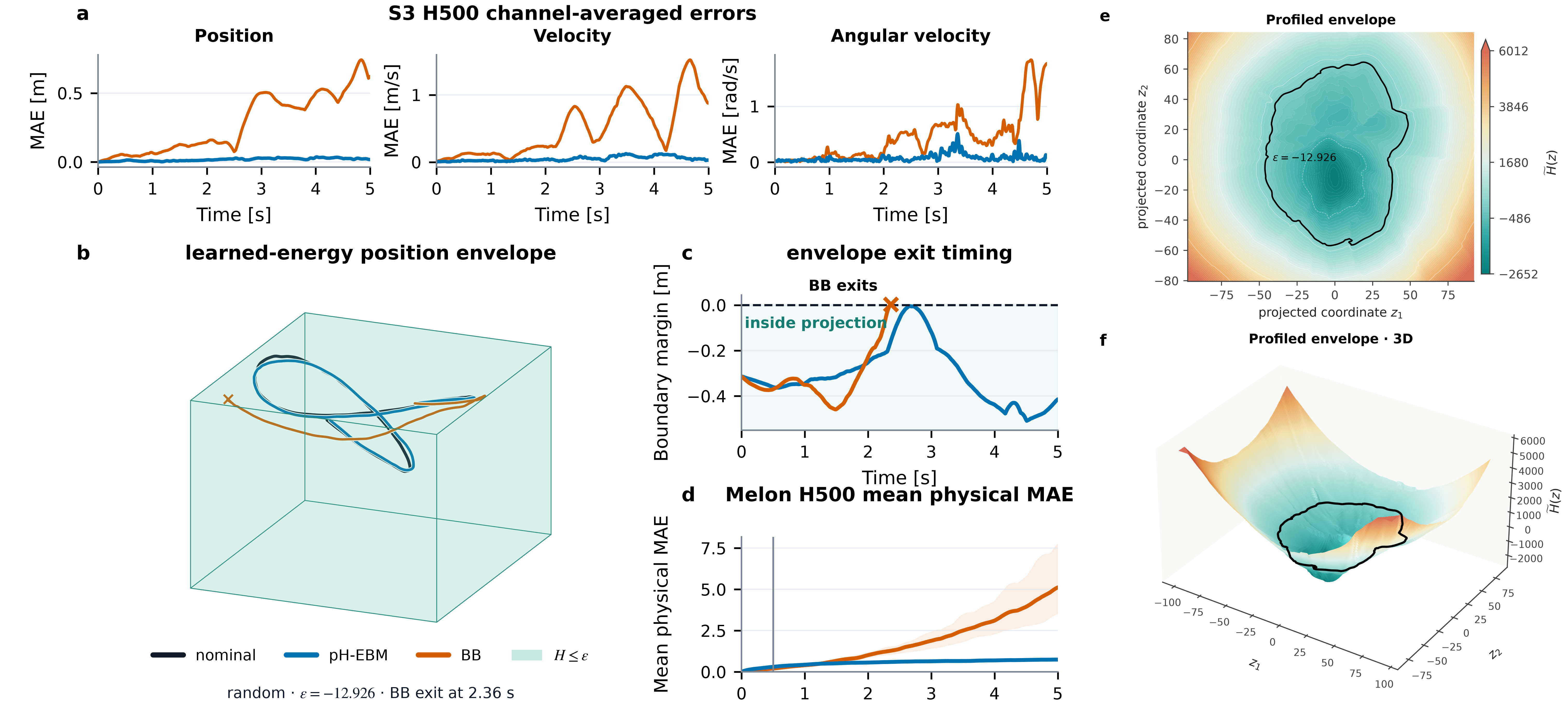}
	\caption{\textbf{\emph{Safe} long-horizon reconstruction of NanoDrone flight dynamics with pH-EBM.}
a) Mean absolute error (MAE) of pH-EBM and the black-box baseline~\citep{BR-CE-FM:26} over a $5,\mathrm{s}$ Random rollout, ten times the $0.5,\mathrm{s}$ training horizon. The pH-EBM maintains substantially lower error, while the BB error rapidly grows beyond the training window.
b-c) Random-trajectory reconstruction relative to the projected certified region $\mathcal C_{\epsilon,\star}$ at $\epsilon=-12.926$. The pH-EBM remains inside the certified set and turns inward near $\Gamma_{\epsilon,\star}$, whereas the BB trajectory exits the region after approximately $2,\mathrm{s}$.
d) Extrapolation to the unseen Melon excitation. Although BB is more accurate over the original $0.5,\mathrm{s}$ horizon, its long-rollout error quickly increases, while pH-EBM remains bounded, highlighting the distinction between short-horizon accuracy and safe recursive deployment.
e-f) Two- and three-dimensional projections of the learned Hamiltonian and shell $\Gamma_{-12.926,\star}$, revealing a strongly nonconvex yet radially unbounded energy landscape that supports certified long-horizon evolution.}
	\label{fig:nano}
\end{figure} 

\paragraph{Robust long time-horizon rollouts in high-dimensional Nanodrone application.} The NanoDrone benchmark considers identification of flight dynamics from collections of short, disjoint $0.5,\mathrm{s}$ trajectories acquired under markedly different excitation distributions (Random, Chirp, Square, and Melon). We evaluate the pH-EBM against the reference black-box neural model of~\citet{BR-CE-FM:26} under two increasingly demanding settings: \textbf{(i)} long-horizon prediction over the training-supported S3 distributions (Random, Chirp, and Square), and \textbf{(ii)} long-horizon prediction under the unseen Melon distribution, corresponding to the original extrapolation task. In both cases, models are rolled out for $5,\mathrm{s}$, i.e., ten times beyond the duration of the trajectories available during training. This setting therefore tests not only whether the learned dynamics fit short flight segments, but whether they remain predictive and operationally well behaved when recursively propagated far beyond the training horizon. On the S3 distributions, the pH-EBM achieves substantially lower prediction error than the reference black-box model, including state-of-the-art short-horizon validation accuracy together with markedly improved long-horizon reconstruction. Figure~\ref{fig:nano}.a shows that its mean absolute error (MAE) remains consistently smaller throughout the $5,\mathrm{s}$ rollout, demonstrating that the structured Hamiltonian dynamics retain the expressivity required for high-dimensional flight identification without sacrificing long-term stability. More importantly, this predictive advantage is accompanied by an explicit safety certificate. Figures~\ref{fig:nano}.(b-c) show the learned trajectory evolving inside the certified energy component $\mathcal C_{\epsilon,\star}$: as it approaches the boundary $\Gamma_{\epsilon,\star}$, the pH-EBM dynamics remain inward-pointing and the trajectory does not leave the certified operating region. In contrast, the black-box NODE accurately tracks the Random trajectory over approximately the $0.5,\mathrm{s}$ training horizon, but subsequently drifts away and eventually exits the same projected safe region. Figures~\ref{fig:nano}.(e-f) visualize the corresponding two- and three-dimensional projections of the learned Hamiltonian and its certified boundary shell, making the relation between the learned energy geometry and the long-horizon trajectory explicit. The unseen Melon distribution exposes a complementary and particularly relevant trade-off. Over the original short-horizon extrapolation task, the pH-EBM is less accurate than the reference black-box model. Yet this ordering reverses qualitatively as the rollout is extended: while the black-box prediction error grows rapidly and the trajectory develops an unstable mode, the pH-EBM remains confined to the certified region and its average MAE approaches a bounded long-horizon regime (Fig.~\ref{fig:nano}.d). The experiment therefore separates two notions that short-horizon benchmark scores alone cannot distinguish: \emph{local predictive fitness} and \emph{safe recursive deployment}. A model may provide the smaller error immediately outside the training distribution while still generating unstable trajectories when propagated for longer horizons. The pH-EBM instead sacrifices some short-horizon Melon accuracy while retaining an \emph{a priori} forward-invariance certificate and preventing the unstable long-term behavior observed in the unconstrained NODE. For safety-critical flight dynamics, where prediction is repeatedly propagated and instability can be more consequential than a moderate increase in instantaneous error, this distinction is central.

\begin{figure}[!h]
	\center
	\includegraphics[width=.9\linewidth]{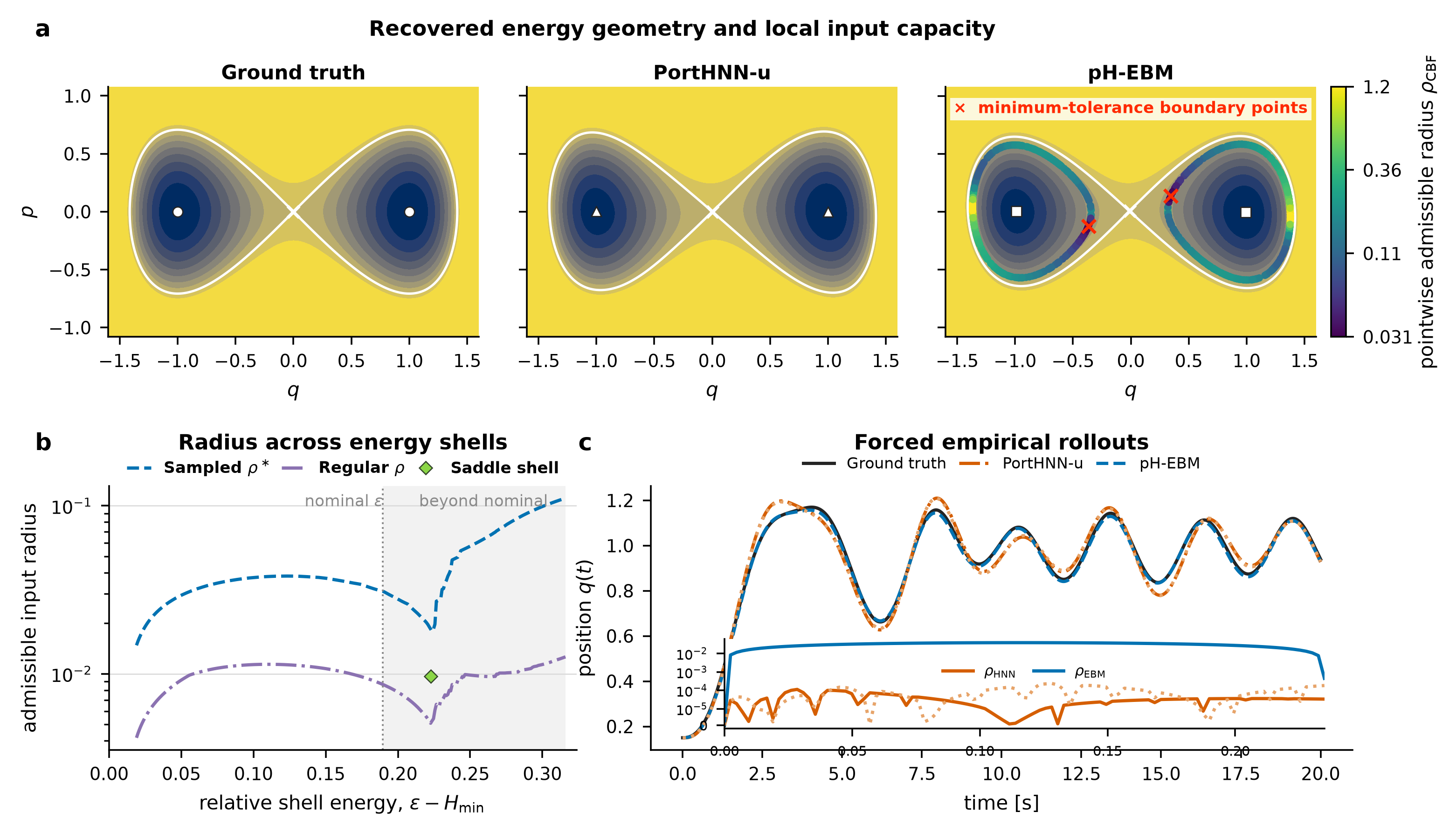}
	\caption{\textbf{Geometry reconstruction, accuracy, and certified robustness in the 2D Duffing oscillator.}
a) Ground-truth Duffing Hamiltonian (left) and energies reconstructed by portHNN-u (center) and pH-EBM (right). Both recover the characteristic nonconvex double-well geometry. The pH-EBM panel additionally reports the pointwise input tolerance, whose minimum occurs near the saddle separating the two wells.
b) Theoretical (purple) and empirical (blue) state-independent radii $\uprho_{\epsilon,\star}$ across an energy-level sweep. The certificate decreases as the shell approaches the saddle energy, then increases again after the certified component expands beyond it, revealing reduced input tolerance near transitions between operating regimes.
c) Reconstruction error and empirical robustness radius for portHNN-u and pH-EBM. While both models achieve accurate trajectory reconstruction, the pH-EBM yields certified radii approximately $2$-$3$ orders of magnitude larger than portHNN-u.}
	\label{fig:duffing}
\end{figure}

\paragraph{\textbf{Nonconvex geometry and attractor-dependent certification.}}
The Duffing oscillator~\citep{MFJ-LM-FS:12} provides a canonical controlled system with a nonconvex double-well potential, where external forcing drives transitions between distinct operating regimes. We compare the pH-EBM against the port-Hamiltonian portHNN-u neural architecture of~\citet{DSA-MM-SD:21}, evaluating both predictive reconstruction and the geometry of the resulting $\uprho$-safety certificate. In the non-chaotic regime, both models recover the characteristic double-well energy landscape from input-output trajectories (Fig.~\ref{fig:duffing}.a), showing that accurate identification can preserve the underlying nonconvex structure. The pH-EBM further reveals how certification depends on the learnt geometry and NODE specifications. The pointwise radius $\uprho_{\mathrm{BF}}(z)$ decreases near the saddle separating the two wells, and the state-independent radius $\uprho_{\epsilon,\star}$ correspondingly weakens as the certified shell approaches the critical energy before increasing again beyond it (Fig.~\ref{fig:duffing}.b). Saddle-proximity weakenanig of the certificate directly illustrates the theoretical link between critical energy levels and reduced input tolerance. Importantly, similar predictive reconstructions do not imply comparable safety margins: under the adopted configuration, the pH-EBM attains a state-independent certificate approximately two orders of magnitude larger than portHNN-u (Fig.~\ref{fig:duffing}.c). Robustness is therefore governed not by prediction error alone, but by the energy geometry and dissipation structure induced by the learned architecture.

\paragraph{Out-of-distribution robustness and mechanical scaling.}
The $n$-link pendulum provides a challenging nonlinear mechanical benchmark whose complexity increases rapidly with the number of coupled joints. For $n=2$ and $n=3$, we compare the pH-EBM against the naive and dissipative neural ODEs studied in~\citet{OY-KR:25} and the portHNN-u model. Across these systems, the pH-EBM improves on the reference models in three key respects: predictive accuracy, robustness to adversarial inputs, and training efficiency (Appendix~\ref{par:nlink}). In particular, under both Brownian and step out-of-distribution forcing, the pH-EBM combines lower prediction error with a certified $\uprho$-robust invariant region. Moreover, training requires approximately $1/20$ of the time for the dissipative reference, illustrating that stability and safety need not come at the computational cost of externally imposed dissipativity mechanisms.

\section{Conclusions}

\paragraph{Limitations of the pH-EBM framework.}
The proposed port-Hamiltonian energy-based model combines expressive nonlinear dynamics with explicit bounded-rollout and safety certificates, but introduces two practical costs: training complexity and hyperparameter sensitivity. In our experiments, pH-EBMs generally require longer training and more targeted model selection than unconstrained NODEs, although their computational cost remains comparable to portHNN-u and can be substantially lower than that of alternative stable architectures such as dissipative NODEs. A more fundamental limitation concerns the scope of the safety guarantee: the certificates derived here apply directly to the \emph{learned model dynamics}, not automatically to the unknown physical plant. Extending invariance guarantees to the true system requires a validated enclosure of model mismatch and external disturbances, which must then be incorporated explicitly into the robust barrier condition. Reducing this model-to-plant gap, together with the training and model-selection burden, is therefore central to future deployment in safety-critical settings.

\paragraph{Future directions.}
The pH-EBM provides a general substrate for learning nonlinear dynamics while retaining stability and quantitative safety guarantees. A natural next step is to extend the framework to physical interactions not explicitly modeled here, including contact and hybrid dynamics, and to integrate the learned model with feedback controllers rather than treating external inputs as prescribed signals. Developing scalable certificate-aware training and closed-loop safety guarantees would move the framework from certified system identification toward end-to-end deployment in safety-critical applications such as autonomous flight, robotic manipulation, and human--robot interaction.

\paragraph{AI usage disclosure.}
Large language models, including ChatGPT and Claude, were used to assist with language polishing, codebase organization and optimization, and the implementation of dataset interfaces for the experimental framework. They were not used for research ideation, formulation of the scientific contributions, development of the theoretical results, or derivation and verification of mathematical proofs. All scientific content, experimental design, implementation choices, and reported results were reviewed and validated by the authors.

\paragraph{Reproducibility statement.}
We provide extensive material to facilitate reproduction of both the theoretical and experimental results. The Appendix states the assumptions underlying the theoretical guarantees and contains complete proofs of the main results. It also documents the datasets, preprocessing, model architectures, training protocols, hyperparameters, evaluation procedures, and certificate computations used throughout the experiments. The accompanying supplementary material includes the implementation of the pH-EBM framework together with the experiment configurations and scripts required to reproduce the reported benchmarks and robustness analyses.

\paragraph{Acknowledgements.}
This work was supported by the IT4LIA AI Factory projects EHPC-AIF-2026PG01-537, EHPC-AIF-2026PG01-1086, and EHPC-AIF-2026PG02-648.

\newpage
\bibliography{SB-main}
\bibliographystyle{unsrtnat}

\newpage
\appendix

\section{Preliminaries}
\label{sec:prel}
\noindent \textbf{Notation.} 
$\vectorones[d]$ denotes the $d$-dimensional vector with all ones, $\vectorzeros[d]$ the $d$-dimensional vector with all zeros, while $\mathcal{I}_d$ denotes the identity matrix in $\real^{d\times d}$. We denote compact subsets of $\real^d$ as sets $\mathcal{B}\ssubset\real^d$, and their boundary as $\partial\mathcal{B}$. For two real vectors $x,y$ of the same dimension, $x^\top y$ denotes the standard inner product. Let $f:\real^d\to\real$; for $f\in C^k(\real^d)$, the function is $k$-times continuously differentiable. For $f\in C^2(\real^d)$, the gradient of $f$ is denoted as $\nabla f$, and the Hessian as $D(\nabla f)=D^2f$. The partial derivative of $f$ with respect to the variable $x_{i}$ is denoted as  $\partial f/\partial{x_{i}}$. $\Theta$ bounds the growth of $f\sim\Theta(g)$ through the existence of constants $a,b>0$ such and a function $g$ such that $a\ g(x)< f(x) < b\ g(x)$. Given functions $g:\real^n\to\real^d$ and $f:\real^d\to\real$, we denote the composition of the two functions as $f\circ g(y)$, for $y\in\real^n$. The abbreviation a.e. stands for almost everywhere, that is for all $\mathcal{B}\subset \real^d$ except zero measure sets. Given a matrix $A\in\real^{d\times d}$, we denote with $A^\top$ its transpose. In case $A$ is symmetric, $\uplambda_{\min}(A),\ \uplambda_{\max}(A)\in\real$ denote its minimum and its maximum eigenvalue. We denote with $B_r(x)\subset \real^d$ the ball of radius $r>0$ and center $x\in\real^d$.

\subsection{Lie and Clarke's derivatives}\label{subsec:gen_der}
\begin{definition}[Lie derivative]\label{def: lie}
    Let $X:\real^{d}\to\real^{d}$ be a locally Lipschitz continuous vector field and let $\Phi_{X}:\real_{\geq 0}\to\real^{d}$ denote the associated local flow. For a smooth function $g:\real^{d}\to\real$, the Lie derivative of $g$ along $X$ is defined as
    \begin{equation}\label{eq: lie-der}
        \mathcal{L}_{X}g(x) = \frac{d}{dt}|_{t=0} (g\ \circ \Phi_{X}^{t} )(x),
    \end{equation}
    and coincides with the directional derivative
    \begin{equation}\label{eq: dir-der}
        \mathcal{L}_{X}g(x)=\nabla_{X}g(x)=\nabla_{x}g(x)^\top X(x).
    \end{equation}
\end{definition}
Lie derivatives~\citep[Chapters 3,~9]{LJM:12} will be central for the characterization of the negative definiteness of the \emph{energy} function along the trajectories generated by the EBM dynamics. We extend the notion of derivatives to nonsmooth functions by recalling Clarke's generalized gradient and directional derivative~\citep{CFH:75}, which apply to any locally Lipschitz function.
\begin{definition}[Generalized gradient]\label{def: gen-grad}
Let $f:\real^d\to\real$ be a locally Lipschitz function. The generalized gradient of $f$ at $x\in\real^d$ is denoted as $\partial f(x)$ and is the convex hull of the set of limits
\begin{equation}
	\lim_{k\to +\infty} \nabla f(x+h_k),
\end{equation}
with $h_k\xrightarrow[]{k\to +\infty} \vectorzeros[d]$ and $f$ differentiable at $x+h_k\in\real^d$ for all $k\in\mathbb{d}$.
\end{definition}
\begin{definition}[Generalized directional derivative]\label{def: gen-dir}
Let $f:\real^d\to\real$ be a locally Lipschitz function and take $v\in\real^d$. Then the Clarke's generalized directional derivative at $x\in\real^d$ is defined as
\begin{align}
	f^{\circ}(x;v)&=\lim_{h\to\vectorzeros[d]}\sup_{\updelta\to 0} \frac{f(x+h+\updelta v)-f(x+h)}{\updelta}\\
	&=\max_{\upxi\in\partial f(x)} \upxi^\top v\label{eq: gen-dir}.
\end{align}
\end{definition}
In particular, combining~\eqref{eq: dir-der} with~\eqref{eq: gen-dir}, the Clarke Lie derivative 
of a locally Lipschitz function $f$ along a vector $v$ is given by
\begin{equation}\label{eq: gen-Lie}
    \mathcal{L}_v f(x)
        := f^{\circ}(x;v)
        = \max_{\xi\in\partial f(x)} \xi^\top v,
\end{equation}
extending the classical smooth definition to the nonsmooth setting.

\section{Technical derivations}
\label{sec:appendix}

\subsection{Port-Hamiltonian energy balance}
\label{app:ph}
A port-Hamiltonian model~\citep{VDSA:07} is a dynamical system characterized by dynamics
\begin{align}\label{eq:gen_port}
	\dot x &= [J(x)-R(x)]\nabla\mathcal{H}+G(x)u+w\\
	y_p &= G(x)^\top\nabla\mathcal{H} 
\end{align}
where $x\in\real^d$ is the state of the model, $J(x)=-J(x)^\top\in\real^{d\times d}$ is the skew-symmetric interconnection matrix, $0\preceq R(x)=R(x)^{\top}\in\real^{d\times d}$ is the symmetric positive-semidefinite dissipation matrix, $G(x)\in\real^{d\times k}$ is the input port, $u\in\real^{k}$ the external time varying input, $w\in\real^d$ is an unknown but bounded disturbance, and $y_p\in\real^{k}$ the power-conjugate output port. Along~\eqref{eq:ph_dynamics},
\begin{align}
    \mathcal{L}_f{\mathcal H}
    &=\nabla\mathcal{H}^\top(J-R)\nabla\mathcal{H}+\nabla\mathcal{H}^\top (Gu+w)\\
    &=\nabla\mathcal{H}^\top J\nabla\mathcal{H}-\nabla\mathcal{H}^\top R\nabla\mathcal{H}+y_{\mathrm p}^\top u+\nabla\mathcal{H}^\top w.
\end{align}
For every real vector $q$, skew-symmetry gives
\begin{equation}
    q^\top J q
    =(q^\top J q)^\top
    =q^\top J^\top q
    =-q^\top J q,
\end{equation}
so $q^\top J_{\uptheta_J}q=0$. Therefore, dissipativity reduces to
\begin{equation}
	\mathcal{L}_f{\mathcal{H}}=-\nabla\mathcal{H}^\top R\nabla\mathcal{H}+y_p^\top u + \nabla\mathcal{H}^\top w.
\end{equation} 
No convexity property of $\mathcal H$ is required for dissipativity. Coercivity is used instead to make energy sublevel sets compact and thereby preclude escape to infinity along bounded-energy trajectories. Convergence to a particular minimum requires additional information on the invariant set where the dissipation vanishes and is not implied by coercivity alone.

\subsection{Derivation and coercivity of the hybrid Hamiltonian}
\label{app:hybrid}
The original energy function for a $L$-layered modern Hopfield network has been introduced in~\citep{KD-HJJ:20,HB-CDH-SH:22}. We identify the visible state with $g_{1}=z\in\real^d$ and denote the hidden variables by $g_{h}\in\real^{N_h}$, $h=2,\ldots,L$. For each layer let $\mathcal F_h:\real^{N_h}\to\real$ be a proper convex differentiable potential and define the activation $\Psi_h(g_{h})=\nabla\mathcal F_h(g_{h})$. Adjacent layers are coupled by matrices $W_{h(h+1)}\in\real^{N_h\times N_{h+1}}$ with $W_{(h+1)h}=W_{h(h+1)}^\top$ and each layer has bias term $b_h\in\real^{N_h}$. The modern Hopfield energy is then
\begin{align}
    \mathcal E(g_{1},\ldots,g_{L})
    ={}&-\sum_{h=1}^{L-1}
    \Psi_h(g_{h})^\top W_{h(h+1)}\Psi_{h+1}(g_{h+1})\nonumber\\
    &+\sum_{h=1}^{L}
    \left[g_{h}^\top\left(\Psi_h(g_{h})-b_h\right)-\mathcal F_h(g_{h})\right]\\
    =& \:\: g_{1}^\top \Psi_1(g_{1})-\mathcal{F}_1(g_{1})-g_{1}^\top b_1\nonumber\\
    &+\sum_{h=2}^{L}\left[-\Psi_{h-1}(g_{h-1})^\top W_{(h-1)h}\Psi_{h}(g_{h})+g_{h}^\top\left(\Psi_h(g_{h})-b_h\right)-\mathcal F_h(g_{h})\right].
    \label{eq:full_layered_energy_app}
\end{align}
The recurrence of the dynamics, with all hidden layers taking contributions from both the layers above and below, makes training the full dynamical architecture a computational and temporal challenge. For this reason, in~\cite{BS-LL:26} a hybrid modern Hopfield energy has been introduced, where only the visible layer maintains its dynamical characterization and all the hidden layers are expressed as static maps $\upphi_h$ of the layer below. Specifically, by expressing $x_{(h)}=\upphi_h(x_{(h-1)})$ for $h=2,\dots,L$, both the energy and the autonomous energy-gradient contribution $\dot x\propto-\nabla \mathcal{E}(x)$ can be expressed in terms of the visible layer state. 

With $g_{1}=z$, the visible layer internal contribution to the Hamiltonian is
\begin{align}
    z^\top\Psi_1(z)-\mathcal F_1(z)&=z^\top z - \frac{1}{2}\|z\|_2^2\\
    &=\frac12\|z\|_2^2.
\end{align}
For every hidden layer $h\geq2$, the feedforward constraint~\eqref{eq:hidden_features} yields
\begin{align}
    g_{h}^\top\Psi_h(g_{h})
    &=\Psi_{h-1}(g_{h-1})^\top
      W_{h(h-1)}^\top\Psi_h(g_{h})\\
    &=\Psi_{h-1}(g_{h-1})^\top
      W_{(h-1)h}\Psi_h(g_{h}),
\end{align}
which cancels exactly the coupling term between layers $h-1$ and $h$ in~\eqref{eq:full_layered_energy_app}. Applying this cancellation for $h=2,\ldots,L$ leaves precisely
\begin{align}
	\mathcal{E}(z,g_{2},\dots,g_{L})&\equiv \mathcal{H}_{\uptheta_H}(z)\nonumber\\
	&=\frac{1}{2}\|z\|_2^2-z^\top b_1-\sum_{h=2}^L\left(\mathcal{F}_h(g_{h}(z))+g_{h}(z)^\top b_h\right)
\end{align}
and the associated neural network action has autonomous contribution
\begin{equation}
	\dot z \propto -\nabla \mathcal{H}_{\uptheta_H}(z)
\end{equation}

We next prove~\eqref{eq:coercive_lower_bound}. Suppose the first hidden activation has bounded image: there exists $M_2<+\infty$ such that $\|\Psi_2(x)\|_2\leq M_2$ for all $x\in\real^{N_2}$. Since $\Psi_2=\nabla\mathcal F_2$, the mean-value inequality gives
\begin{equation}
    |\mathcal F_2(z)-\mathcal F_2(0)|\leq M_2\|z\|_2,
\end{equation}
so
\begin{equation}
    \mathcal F_2(W_{21}z)
    \leq |\mathcal F_2(0)|+M_2\|W_{21}\|_2\|z\|_2.
    \label{eq:F2_linear_growth}
\end{equation}
Moreover, $g_{3}=W_{32}\Psi_2(g_{2})$ belongs to a compact set independent of $z$. By continuity of $\Psi_3$, its image over this compact set is bounded; recursively, every $g_{h}(z)$ for $h\geq3$ lies in a compact set. Continuity of $\mathcal F_h$ and compactness of the layer variables $g_{h}$ then yields constants $C_h<+\infty$ such that
\begin{equation}
    \mathcal F_h(g_{h}(z))+g_{h}(z)^\top b_h\leq C_h,
    \qquad h=3,\ldots,L.
\end{equation}
Substituting these estimates into~\eqref{eq:hybrid_hamiltonian} yields
\begin{equation}
    \mathcal H_{\uptheta_H}(z)
    \geq \frac12\|z\|_2^2-c_1\|z\|_2-c_0
    \label{eq:coercive_lower_bound}
\end{equation}
with
\begin{equation}
    c_1=\|W_{21}\|_2(M_2+b_2)+b_1,
    \qquad
    c_0=|\mathcal F_2(0)|+\sum_{h=3}^{L}C_h.
\end{equation}
The quadratic term dominates the affine lower bound as $\|z\|_2\to+\infty$, proving coercivity.

\subsection{Barrier functions and invariance}
\label{app:barrier_primer}
Consider a controlled system $\dot z=f(z,u,w)$ and a continuously differentiable function $h:\mathbb R^d\to\mathbb R$. The superlevel set
\begin{equation}
    \mathcal C:=\{z:h(z)\geq0\}
\end{equation}
is \emph{invariant} if every solution initialized in $\mathcal C$ remains in $\mathcal C$ for all future times for which the solution exists. A standard way to certify this property is through a zeroing barrier function. Let $\alpha$ be an extended class-$\mathcal K$ function, i.e., a continuous strictly increasing function satisfying $\alpha(0)=0$. For a fixed admissible input/disturbance pair, the inequality
\begin{equation}
    \mathcal L_f h(z,u,w)+\alpha(h(z))\geq0
    \label{eq:barrier_primer}
\end{equation}
prevents $h$ from decreasing through zero. Under the usual solution-existence assumptions, if~\eqref{eq:barrier_primer} holds on a neighborhood of $\mathcal C$, then $\mathcal C$ is invariant. If it holds for every $u$ in a prescribed input set and every $w$ in a prescribed disturbance set, the same set is robustly invariant for those exogenous signals~\citep{AAD-XX-GJW:17,AAD-CS-EM:19}.

On a regular boundary point $z\in\partial\mathcal C$, where $h(z)=0$ and $\nabla h(z)\neq0$, the condition reduces to the geometric inward-pointing requirement
\begin{equation}
    \mathcal L_f h(z,u,w)\geq0.
\end{equation}
Equivalently, the vector field must lie in the tangent cone of $\mathcal C$; this is the local form of Nagumo's invariance condition. 

\subsection{BF invariance, robust input sets, and boundary geometry}
\label{app:input_bounds}
We first distinguish three levels of certification that are used in the main text: a pointwise zeroing-BF condition, a boundary-only invariance condition, and a state-independent set of exogenous inputs that satisfies the boundary condition everywhere. In the following treatment, we will refer to a specific condition being \emph{robustly} satisfied when it holds over the entire support $\mathcal{W}$ of the unknown disturbances.

Let $h_\epsilon(z)=\epsilon-\mathcal H(z)$ and consider the perturbed port-Hamiltonian dynamics
\begin{equation}
    \dot z=[J(z)-R(z)]\nabla\mathcal H(z)+G(z)u+w,
    \qquad w\in\mathcal W(z),
    \label{eq:perturbed_ph_app}
\end{equation}
where $\mathcal W(z)\subset\real^d$ is nonempty and compact. Write $d_H:=\nabla\mathcal{H}^\top R\nabla\mathcal{H}$, and $a_H:=G^\top \nabla\mathcal{H}$. Then
\begin{equation}
    \mathcal{L}_f h_\epsilon
    =d_H-a_H^\top u-\nabla\mathcal{H}^\top w.
    \label{eq:robust_hdot_app}
\end{equation}
Therefore the zeroing-BF inequality $\mathcal{L}_f h_\epsilon+\upalpha(h_\epsilon)\geq0$ for every $w\in\mathcal W(z)$ is equivalent to
\begin{equation}
    a_H(z)^\top u+\upsigma_{\mathcal W(z)}(\nabla\mathcal{H}(z))
    \leq d_H(z)+\upalpha(h_\epsilon(z)).
    \label{eq:robust_cbf_app}
\end{equation}
Setting $\mathcal W(z)=\{0\}$ recovers the result without bounded perturbayions. If $a_H(z)=0$, no division by $\|a_H(z)\|_*$ is performed: the energy effect of the input is identically zero to first order and feasibility depends only on dissipation, the barrier slack, and the disturbance term.

For any norm $\|\cdot\|$ on input space with dual norm $\|\cdot\|_*$, H\"older's inequality gives
\begin{equation}
    a_H^\top u\leq\|a_H\|_*\|u\|.
\end{equation}
Thus~\eqref{eq:pointwise_norm_bound} is sufficient in the nominal case. It is the largest origin-centered norm ball contained in the unconstrained half-space because
\begin{equation}
    \sup_{\|u\|\leq\uprho}a_H^\top u=\uprho\|a_H\|_*.
\end{equation}

\subsection{Exact robust boundary certificate}
\label{app:robust_boundary}
Let $z_\star$ be an isolated local minimum and define
\begin{equation}
    \mathcal C_{\epsilon,\star}
    :=\operatorname{Comp}_{z_\star}\{z:h_{\epsilon,\star}(z)\geq 0\},
    \qquad
    h_{\epsilon,\star}(z):=\epsilon-V_\star(z),
    \qquad
    \Gamma_{\epsilon,\star}:=\partial\mathcal C_{\epsilon,\star}.
    \label{eq:component_safe_set_app}
\end{equation} 
The energy component is both connected (by definition) and compact by coercivity of $\mathcal H$, and the solutions of~\eqref{eq:perturbed_ph_app} exist uniquely for the admissible measurable inputs under consideration.

On $\Gamma_{\epsilon,\star}$ one has by definition that $h_{\epsilon,\star}=0$, therefore the robust inward-pointing condition is
\begin{equation}
    a_H(z)^\top u+\upsigma_{\mathcal W(z)}(\nabla\mathcal{H}(z))\leq d_H(z).
    \label{eq:robust_boundary_halfspace_app}
\end{equation}
Before proceeding with the first theorem, we provide a general notion from differential geometry, that is, that of regular hypersurface, which will be used in reference to the boundaries of our candidate safe sets $\mathcal{C}_{\epsilon,\star}$.
\begin{proposition}[Regular $(d-1)$-dim hypersurface]\label{prop:reg_bound}
Let $f:\real^d\to\real$ be a $C^1$ function, and for $s\in\real$ define the level set
\begin{equation}
	\Gamma_{s}:=\{z\in\real^d:\ f(z)=s\}
\end{equation}
Then we say that $\Gamma_s$ is a regular $(d-1)$-dim $C^1$-hypersurface if 
\begin{equation}
	\nabla f(z)\neq 0\qquad \forall z\in\Gamma_s.
\end{equation}
\end{proposition}
When the Hamiltonian $\mathcal{H}$ has isolated critical points, the shells $\Gamma_{\epsilon,\star}$ are regular almost everywhere $\epsilon>0$, as there will be finitely many values of $\epsilon$ such that a saddle point or local maxima lie in $\Gamma_{\epsilon,\star}$. The existence of entire connected open subsets $\mathcal V\subset \real^d$ where the regularity assumption is not satisfied, and hence such that $\mathcal H(z)=c$ for all $z\in\mathcal V$, is possible but unlikely. The existence of such regions would require the following two conditions to hold simultaneously
\begin{equation}\label{eq:irr_ls}
\nabla \mathcal{H}(z)=0,\quad \operatorname{det} \nabla^2 \mathcal{H}(z)=0\qquad \forall z\in\mathcal V.
\end{equation}
More specifically, if the Hamiltonian $\mathcal{H}$ is constituted only by primitives $\{\mathcal{F}_h\}_{h=1}^L$ of real analytic activation functions - such as the $\operatorname{softmax}$, $\operatorname{tanh}$, $\operatorname{logistic}$, and $\operatorname{softplus}$ - then the existence of open subsets $\mathcal{V}$ with constant energy is impossible due to analytic continuation. If such $\mathcal{V}$ existed, it would imply that $\mathcal{H}\equiv c$ on all $\real^d$, which contraddicts the coercivity of $\mathcal{H}$. Therefore, the condition in~\eqref{eq:irr_ls} can only be satisfied by neural networks characterized by less regular activation functions with broad plateaus - such as $\operatorname{ReLU}$ - and even in such cases remains unlikely.
In the following theorem, we specialize the definition of safe sets as defined by the energy BF to our port-Hamiltonian framework.

\begin{theorem}[Maximal uniform robust input set]
\label{thm:uniform_boundary_set}
Let $\Gamma_{\epsilon,\star}$ be the regular boundary of the zero level set $h_{\epsilon,\star}=0$ in the sense of Proposition~\ref{prop:reg_bound}, and define
\begin{equation}
\label{eq:uniform_robust_set}
	\overline{\mathcal U}_{\epsilon,\star}^{\mathrm{rob}}
    :=\mathcal U\cap
    \bigcap_{z\in\Gamma_{\epsilon,\star}}
    \left\{u:
    a_H(z)^\top u+\upsigma_{\mathcal W(z)}(\nabla\mathcal H(z))\leq d_H(z)\right\}.
\end{equation}
Then
\begin{itemize}
	\item[(i)] for any locally essentially bounded input satisfying
$u(t)\in\overline{\mathcal U}_{\epsilon,\star}^{\mathrm{rob}}$ almost everywhere, $\mathcal C_{\epsilon,\star}$ is robustly invariant for every measurable disturbance $w(t)\in\mathcal W(z(t))$;
	\item[(ii)] $\overline{\mathcal U}_{\epsilon,\star}^{\mathrm{rob}}$ is maximal with respect to set inclusion among state-independent subsets of $\mathcal U$ whose every element satisfies~\eqref{eq:robust_boundary_halfspace_app} at every boundary state.
\end{itemize}
\end{theorem}

\begin{proof}
\begin{itemize}
	\item[(i)] let $z\in\Gamma_{\epsilon,\star}$, and observe that $a_H(z)^\top u+\upsigma_{\mathcal W(z)}(\nabla\mathcal{H}(z))\leq d_H(z)\iff \mathcal{L}_f h_{\epsilon,\star}(z,u,w)\geq0$ for every $w\in\mathcal W(z)$. Intersecting these pointwise admissible half-spaces over the full boundary yields exactly $\overline{\mathcal U}_{\epsilon,\star}^{\mathrm{rob}}$. $\mathcal{L}_f h\geq 0$ on $\Gamma_{\epsilon,\star}$ implies that the neural network action is in the tangent cone of the sublevel component, so Nagumo's tangent condition yields robust invariance.
	\item[(ii)] Let $\mathcal{U}_{\epsilon,\star}^-$ be such that $\mathcal U_{\epsilon,\star}^-/\overline{\mathcal U}_{\epsilon,\star}^{\mathrm{rob}}\neq \emptyset$. Then there exists $u\in\mathcal{U}_{\epsilon,\star}^-$ such that
	\begin{equation}
		a_H(z)^\top u+\upsigma_{\mathcal W(z)}(\nabla\mathcal{H}(z))> d_H(z)
	\end{equation}
	and $\mathcal{U}_{\epsilon,\star}^-$ does not make $\mathcal{C}_{\epsilon,\star}$ robustly invariant.
\end{itemize}
\end{proof}
It is now easy to generalize the previous definition of safety ensured input set to arbitrary geometries instead of normed balls. 

\begin{corollary}[Certification of arbitrary compact input sets]
\label{cor:support_input_set}
Let $\mathcal U_0\subseteq\mathcal U$ be compact. Then
$\mathcal U_0\subseteq\overline{\mathcal U}_{\epsilon,\star}^{\mathrm{rob}}$
if and only if
\begin{equation}
	\upsigma_{\mathcal{U}_0}(a_{H}(z))+\upsigma_{\mathcal{W}(z)}(\nabla\mathcal{H}(z))\leq d_{H}(z)\qquad \forall z\in\Gamma_{\epsilon,\star}
\end{equation}
\end{corollary}
\begin{proof}
Let $\upzeta(z,u)=a_{H}^\top u+\upsigma_{\mathcal{W}(z)}(\nabla\mathcal{H}(z))-d_{H}(z)$. The function $\upzeta$ is continuous in both of its arguments. Since both $\Gamma_{\epsilon,\star}$ and $\mathcal{U}_0$ are compact, then by Weierstrass theorem $\upzeta$ constrained to this sets admits maximum and minimum with respect to both of its arguments.

The certification of arbitrary compact input sets then requires
\begin{align}
	0&<\min_{z\in\Gamma_{\epsilon,\star}}\max_{u\in \mathcal{U}_0} \upzeta(z,u)\nonumber\\
	&=\min_{z\in\Gamma_{\epsilon,\star}} \left[\max_{u\in \mathcal{U}_0} a_{H}^\top u\right]+\upsigma_{\mathcal{W}(z)}(\nabla\mathcal{H}(z))-d_{H}(z)\nonumber\\
	&=\min_{z\in\Gamma_{\epsilon,\star}}\upsigma_{\mathcal{U}_0}(a_{H}(z))+\upsigma_{\mathcal{W}(z)}(\nabla\mathcal{H}(z))-d_{H}(z)
\end{align}
\end{proof}

Define the residual boundary dissipation after worst-case disturbance injection,
\begin{equation}
    \updelta(z)
    :=d_H(z)-\upsigma_{\mathcal W(z)}(\nabla\mathcal{H}(z)).
    \label{eq:residual_boundary_margin}
\end{equation}
Assume $\updelta(z)\geq0$ on the boundary, so that zero input is itself robustly admissible. The exact radius of the largest origin-centered norm ball contained in~\eqref{eq:uniform_robust_set} is
\begin{equation}
    \uprho_{\epsilon,\star}
    :=\inf_{\substack{z\in\Gamma_{\epsilon,\star}\\ a_H(z)\neq0}}
       \frac{\updelta(z)}{\|a_H(z)\|_*},
    \label{eq:exact_robust_radius}
\end{equation}
Boundary points with $a_H(z)=0$ impose no input-radius restriction as long as $\updelta(z)\geq0$. Consequently, every $\uprho\leq\uprho_{\epsilon,\star}$ satisfies
$\{u:\|u\|\leq\uprho\}\cap\mathcal U\subseteq\overline{\mathcal U}_{\epsilon,\star}^{\mathrm{rob}}$.

A frequently useful but more conservative separable bound follows from
\begin{equation}
    \frac{\min_{\Gamma_{\epsilon,\star}}\updelta(z)}
         {\max_{\Gamma_{\epsilon,\star}}\|a_H(z)\|_*}
    \leq
    \inf_{\substack{z\in\Gamma_{\epsilon,\star}\\a_H(z)\neq0}}
    \frac{\updelta(z)}{\|a_H(z)\|_*}.
    \label{eq:separable_radius_lower_bound}
\end{equation}
This inequality corrects the tempting but generally false replacement of the infimum of a ratio by the ratio of independent extrema.

\subsection{Regular energy shells and normal robustness decomposition}
\label{app:regular_boundary}
The regularity condition for the shell has deeper geometric implications, since it allows to clearly outline the vectors resulting in inward and outward motion with respect to the set $\mathcal{C}_{\epsilon,\star}$.

\begin{proposition}[Regular energy boundary]
\label{prop:regular_boundary}
Assume that
\begin{equation}
    \nabla\mathcal H(z)\neq0,
    \qquad z\in\Gamma_{\epsilon,\star}.
    \label{eq:regular_boundary_condition}
\end{equation}
Then every point of $\Gamma_{\epsilon,\star}$ belongs locally to the regular level set
$\{z:h_{\epsilon,\star}=0\}$, which is a $C^1$ embedded hypersurface of dimension $(d-1)$, and we can characterize the tangent space and the tangent cone at $\Gamma_{\epsilon,\star}$ as
\begin{align}
    T_z\Gamma_{\epsilon,\star}
    &=\{v:\nabla\mathcal H(z)^\top v=0\},\\
    T_{\mathcal C_{\epsilon,\star}}(z)
    &=\{v:\nabla\mathcal H(z)^\top v\leq0\}.
    \label{eq:tangent_cone_energy}
\end{align}
In addition, if the boundary is compact, then
$\kappa_{\epsilon,\star}:=\inf_{\Gamma_{\epsilon,\star}}\|\nabla\mathcal H\|_2>0$.
\end{proposition}

\begin{proof}
Consider the BF $h_{\epsilon,\star}(z)$, and its directional derivative
$Dh_{\epsilon,\star}(z)[v]=-\nabla\mathcal H(z)^\top v$. Since $h_{\epsilon,\star}$ maps $\real^d$ to $\real$, the derivative is surjective at all points $z\in\Gamma_{\epsilon,\star}$ such that $\nabla\mathcal{H}(z)\neq0$. Since by assumption $\nabla\mathcal{H}(z)\neq 0$, we have that $h_{\epsilon,\star}=0$ is a submersion and $\Gamma_{\epsilon,\star}$ is a $C^1$ $(d-1)$-hypersurface.

Because $\mathcal C_{\epsilon,\star}$ is the superlevel set $h_{\epsilon,\star}\geq0$, the feasible first-order directions are defined by $Dh_{\epsilon,\star}(z)[v]\geq0$,or equivalently $\nabla\mathcal{H}(z)^\top v\leq 0$, giving~\eqref{eq:tangent_cone_energy}. Finally, by Weierstrass theorem the continuous function $\|\nabla\mathcal H\|_2$ attains a positive maximum and minimum on the compact boundary, which must be strictly positive given the assumption $\nabla\mathcal{H}(z)\neq 0$ on $\Gamma_{\epsilon,\star}$.
\end{proof}

The characterization of tangent space and tangent cone of $\Gamma_{\epsilon,\star}$, and in particular the nonzero-gradient condition is essential for a first-order barrier test. If $\nabla\mathcal{H}=0$ at a boundary point, $\mathcal{L}_f h_{\epsilon,\star}$ vanishes and $Dh_{\epsilon,\star}(z)[v]\equiv 0$ for any vector $v\in\real^d$: therefore, the BF condition fails to distinguish inward from outward motion. Thereby, regularity makes the energy gradient a valid boundary normal for the BF argument.

For $z\in\Gamma_{\epsilon,\star}$ define the outward energy normal
\begin{equation}
    \upeta(z):=\frac{\nabla\mathcal{H}(z)}{\|\nabla\mathcal{H}(z)\|_2}.
\end{equation}
Positive homogeneity of support functions and the identities
\begin{align}
    d_H(z)&=\|\nabla\mathcal{H}(z)\|_2^2 \upeta(z)^\top R(z)\upeta(z),\\
    \|a_H(z)\|_*&=\|\nabla\mathcal{H}(z)\|_2\|G(z)^\top \upeta(z)\|_*,\\
    \upsigma_{\mathcal W(z)}(\nabla\mathcal{H}(z))
    &=\|\nabla\mathcal{H}(z)\|_2\upsigma_{\mathcal W(z)}(\upeta(z))
\end{align}
show that, whenever $G(z)^\top \upeta(z)\neq0$,
\begin{equation}\label{eq:norm_inp}
    \frac{\updelta(z)}{\|a_H(z)\|_*}
    =\frac{\|\nabla\mathcal{H}(z)\|_2 \upeta(z)^\top R(z)\upeta(z)
    -\upsigma_{\mathcal W(z)}(\upeta(z))}
    {\|G(z)^\top \upeta(z)\|_*}.
\end{equation}
Interestingly,~\eqref{eq:norm_inp} reveals that, with respect to the input port, only the directions that are normal to the level set consumes the energy margin, while tangent directions leave it unaffected. In particular, if $\|G(z)^\top \upeta(z)\|_*=0$ for all $z\in\Gamma_{\epsilon,\star}$, the input is tangent to the Hamiltonian shells over the entire boundary and the energy certificate imposes no magnitude restriction on $u$; only the disturbance margin must remain nonnegative.

Let
\begin{align}
    r_{\epsilon,\star}&:=\inf_{\Gamma_{\epsilon,\star}}\upeta^\top R\upeta,\\
    \bar g_{\epsilon,\star}^{\perp}&:=\sup_{\Gamma_{\epsilon,\star}}\|G^\top \upeta\|_*,\\
    \bar w_{\epsilon,\star}^{\perp}&:=\sup_{\Gamma_{\epsilon,\star}}\upsigma_{\mathcal W(z)}(\upeta).
\end{align}
If $\bar g_{\epsilon,\star}^{\perp}>0$, and
$r_{\epsilon,\star}\kappa_{\epsilon,\star}\geq\bar w_{\epsilon,\star}^{\perp}$, then~\eqref{eq:norm_inp} yields
\begin{equation}
    \uprho_{\epsilon,\star}
    \geq
    \frac{r_{\epsilon,\star}\kappa_{\epsilon,\star}
    -\bar w_{\epsilon,\star}^{\perp}}
    {\bar g_{\epsilon,\star}^{\perp}},
\end{equation}
which proves~\eqref{eq:norm_inp}. 

For each attractor component one may define the maximal regular-shell energy
\begin{equation}
    \epsilon_\star^{\mathrm{reg}}
    :=\sup\left\{\bar\epsilon>0:\nabla\mathcal H(z)\neq0\quad 
    \forall z\in\Gamma_{e,\star},\: 
    \forall \epsilon\in(0,\bar\epsilon)\right\}.
    \label{eq:regular_shell_threshold}
\end{equation}
Every compact shell below this threshold is eligible for the regular-boundary certificate. When a shell approaches a critical energy containing a saddle or another stationary point, $\kappa_{\epsilon,\star}$ can collapse toward zero, quantitatively revealing the loss of first-order robustness. After the saddle point, the $\uprho$-safety certificate exit resumes, but the new certified region will contain the sets on "both" sides of the saddle.

\subsection{Bounds under local PL and persistent forcing}
\label{app:pl_bounds}
Assume the component $\mathcal C_{\epsilon,\star}$ lies inside a neighborhood $\mathcal N_\star$ of a local minimum $z_\star\in\real^d$ on which the local Polyak-\L{}ojasiewicz (PL) condition~\citep{KH-NJ-SM:16} holds for the Hamiltonian $\mathcal H$. 
\begin{equation}
    \frac12\|\nabla\mathcal{H}(z)\|_2^2\geq\upmu_\star\left(\mathcal{H}(z)-\mathcal{H}(z_\star)\right)=\upmu_\star V_\star(z),
    \qquad
    d_H(z)\geq r_\star\|\nabla\mathcal{H}(z)\|_2^2,
    \label{eq:pl_app_component}
\end{equation}
for constants $\upmu_\star,r_\star>0$ where $r_\star\leq\min_{z\in\mathcal{C}_{\epsilon,\star}}\uplambda_{\min}(R(z))$. The PL condition is associated to a region of convex growth around a minima, and since we are requiring it to be satisfied for all points in $\mathcal{N}_\star$, then necessarily we will have that $\mathcal{N}_\star\subset \mathcal{C}_{\epsilon_\star^{\text{reg}},\star}$, where $\epsilon_\star^{\text{reg}}$ has been defined in~\eqref{eq:regular_shell_threshold}. Indeed, it holds that
\begin{equation}
	\lim_{\epsilon\to\epsilon_\star^{\text{reg}}} \kappa_{\epsilon,\star}=0
\end{equation}
and $\mathcal{N}_\star\supset \mathcal{C}_{\epsilon_\star^{\text{reg}},\star}$ would lead, for a point $z\in\Gamma_{\epsilon_\star^{\text{reg}},\star}$ to the contraddiction
\begin{equation}
	0\geq \underbrace{\upmu_\star V_\star(z)}_{>0}
\end{equation} 

Let $R(x)\succ 0$ over the entire component $\mathcal{C}_{\epsilon,\star}$. Then the PL condition implies
\begin{equation}
    d_H(z)\geq2r_\star\upmu_\star V_\star(z).
    \label{eq:pl_dissipation_chain}
\end{equation}

For the zeroing-BF condition inside the component,~\eqref{eq:pl_dissipation_chain} show that
\begin{equation}
    a_H(z)^\top u
    \leq2r_\star\upmu_\star V_\star(z)
    +\upalpha\bigl(\epsilon-V_\star(z)\bigr)
    \label{eq:pl_inner_set}
\end{equation}
is a sufficient conservative approximation of the exact admissible half-space. If $\upalpha(s)=\upgamma s$, the right-hand side is affine in $q=V_\star(z)\in[0,\epsilon]$. Considering the first derivative of $l(q)=2r_\star\upmu_\star q+\upgamma(\epsilon-q)$, we have $\dot l(q)=2r_\star\upmu_\star-\upgamma$. Consequently, if $2r_\star\upmu_\star>\upgamma$, then $l$ is strictly increasing over $[0,\epsilon]$ and its minimum on the interval is $\epsilon\upgamma$. Conversely, if $\upgamma>2r_\star\upmu_\star$, the $l$ is strictly decreasing over $[0,\epsilon]$ and its minimum is $2r_\star\upmu_\star\epsilon$. Therefore, the minimum of the r.h.s. can be expressed as
\begin{equation}
    \upbeta_\epsilon=\epsilon\min\{2r_\star\upmu_\star,\upgamma\}.
\end{equation}
Consequently, with
$A_\epsilon:=\sup_{z\in\mathcal C_{\epsilon,\star}}\|a_H(z)\|_*<\infty$, the bound
\begin{equation}
    \|u\|\leq\frac{\upbeta_\epsilon}{A_\epsilon}
    \label{eq:uniform_pl_norm_bound}
\end{equation}
ensures the chosen zeroing-BF inequality throughout the component whenever $A_\epsilon>0$. This is stronger than boundary-only invariance and is correspondingly more conservative.

On $\Gamma_{\epsilon,\star}$, PL yields
\begin{equation}
    \|\nabla\mathcal{H}(z)\|_2\geq\sqrt{2\upmu_\star\epsilon}>0.
    \label{eq:boundary_gradient_lower}
\end{equation}
Hence the boundary is regular by Proposition~\ref{prop:regular_boundary} and
$\kappa_{\epsilon,\star}\geq\sqrt{2\upmu_\star\epsilon}$. Substitution in the normal robustness bound gives the robust PL corollary
\begin{equation}
    \uprho_{\epsilon,\star}
    \geq
    \frac{r_{\epsilon,\star}\sqrt{2\upmu_\star\epsilon}
    -\bar w_{\epsilon,\star}^{\perp}}
    {\bar g_{\epsilon,\star}^{\perp}},
    \label{eq:pl_robust_boundary_radius}
\end{equation}
whenever the numerator is nonnegative. This is an invariance-only certificate; it need not satisfy the selected zeroing-BF inequality at every interior point.

\subsection{Input-to-energy robustness under persistent forcing}
\label{app:input_energy}
In this subsection we are going to present an additional bound on the confinement of the port-Hamiltonian trajectories $z(t)$ whenever they are in a component $\mathcal{C}_{\epsilon,\star}$ and we have uniform bounds on the input $u$, the input-port $G$ and the disturbances $w$.

\begin{proposition}[Input-to-energy tube]
\label{thm:input_energy}
Let the uniform bounds 
\begin{equation}\label{eq: unif_bound}
    \|G(z)\|_2\leq\bar g,\qquad
    \|u(t)\|_2\leq\bar u,\qquad
    \|w(t)\|_2\leq\bar w.
\end{equation}
hold in the component $\mathcal{C}_{\epsilon,\star}$  and set $c:=\bar g\bar u+\bar w$. If
\begin{equation}
    c\leq r_\star\sqrt{2\upmu_\star\epsilon},
    \label{eq:input_energy_invariance_condition}
\end{equation}
then $\mathcal C_{\epsilon,\star}$ is robustly invariant and every trajectory initialized in the component satisfies~\eqref{eq:input_energy_estimate}. In particular,
\begin{equation}
    \limsup_{t\to\infty}V_\star(z(t))
    \leq\frac{c^2}{2r_\star^2\upmu_\star}.
    \label{eq:ultimate_energy_tube}
\end{equation}
\end{proposition}

\begin{proof}
 	Along the trajectories of the latent port-Hamiltonian EBM model, defined by dynamics~\eqref{eq:perturbed_ph_app} with vector field $f(z)=[J(z)-R(z)]\nabla\mathcal{H}(z)+G(z)u+w$, we have the following Lie derivative on the difference function $V_\star$.
\begin{align}\label{eq:Lyap_dis}
	\mathcal{L}_f V_\star(z,u)&=\mathcal{L}_f \mathcal H\: (z,u)\nonumber\\
	&=-\nabla\mathcal{H}(z)R(z)\nabla\mathcal{H}(z)+\left(G(z)u+w\right)^\top\nabla\mathcal{H}(z)\nonumber\\
	&\leq -r_\star\|\nabla\mathcal{H}(z)\|_2^2+c\|\nabla\mathcal{H}(z)\|_2\nonumber\\
	&\leq -\frac{r_\star}{2}\|\nabla\mathcal{H}(z)\|_2^2+\frac{c^2}{2r_\star}\nonumber\\
	&\leq -r_\star \upmu_\star V_\star(z)+\frac{c^2}{2r_\star}
\end{align}
where in the fourth passage we have used Young's inequality $ab\leq \varepsilon a^2/2+b^2/(\varepsilon 2)$ with $a=\|\nabla\mathcal H\|$, $b=c$, and $\varepsilon=r_\star$ and in the last passage the PL condition. 

Under Assumption~\ref{eq:input_energy_invariance_condition} we actually have on the level set $h_{\epsilon,\star}=0$
\begin{align}
	-\mathcal{L}_f h_{\epsilon,\star}\: (z,u)&=\mathcal{L}_f V_\star\: (z,u)\nonumber\\
	&\leq -r_\star \upmu_\star\epsilon + \frac{(r_\star\sqrt{2\upmu_\star\epsilon})^2}{2r_\star}=0
\end{align}
and consequently we also have that $\mathcal{C}_{\epsilon,\star}$ is robustly invariant for all $(u,w)$ satisfying Assumption~\ref{eq: unif_bound}. Finally, applying the comparison lemma on~\eqref{eq:Lyap_dis} we obtain
\begin{equation}
    V_\star(t)\leq e^{-r_\star\upmu_\star t}V_\star(0)
    +\frac{c^2}{2r_\star^2\upmu_\star}\left(1-e^{-r_\star\upmu_\star t}\right).
\end{equation}
Taking the limit $t\to +\infty$ we conclude.
\end{proof}
Notice finally that when $c<r_\star\sqrt{2\upmu_\star\epsilon}$, we are actually stating that asymptotically the trajectories of the system will be bounded to level sets $h_{\epsilon,\star}=s$ with $s\geq\epsilon-c^2/(2r_\star\upmu_\star)$.

\subsection{Activation-dependent guidance for local PL}
\label{app:activation_pl}
We begin the current subsection by proving the general result of Proposition~\ref{prop:hessian_pl}, and the proceed to specialize the result to an example architecture and different classes of activation functions.

\begin{proof}
Since $\mathcal{H}\in C^{1,1}_{\text{loc}}(\real^d)$, then the Hessian $\nabla^2\mathcal{H}(z)$ is defined a.e. $z\in\real^d$, and at the points of non-differentiability we relay on Clarke's generalized Jacobian $z\mapsto \partial(\nabla\mathcal{H})(z)$ as defined in Definition~\ref{def: gen-grad}. The set $\partial(\nabla\mathcal{H})(z)$ is nonempty, compact, convex, and upper semi-continuous, and consequently
\begin{equation}
	Q(z)=Q(z)^\top\qquad \forall Q(z)\in \partial(\nabla\mathcal{H})(z).
\end{equation}
Let $m_\star=\min_{Q(z)\in\partial(\nabla\mathcal{H})(z)} \uplambda_{\min} (Q(z))$, and since $z_{\star}$ is a local minimum, by compactness we have that $Q(z)\succ 0$ for all $Q(z)\in\partial(\nabla\mathcal{H})(z)$ and consequently $m_\star>0$.

By upper semicontinuity of the generalized Jacobian there exists $\updelta>0$ and a neighbourhood $\mathcal{N}_{\updelta}$ of the local minimum $z_\star$ such that
\begin{equation}
	\|Q(z)-Q(z_\star)\|_2\leq \updelta \qquad \forall z\in\mathcal{N}_{\updelta}
\end{equation}
Fix $\upmu_\star\in(0,m_\star)$ such that $\updelta<m_\star-\upmu_\star$. Then is holds that 
\begin{equation}\label{eq:matr_bound}
	Q(z)\succ \upmu_\star I_d\qquad \forall Q(z)\in\partial(\nabla\mathcal{H})(z),\: \forall z\in \mathcal{N}_\updelta  
\end{equation}
By Taylor expansion of $\mathcal{H}$ at $z_\star$ in $\mathcal{N}_\updelta$ we obtain
\begin{align}
	\mathcal{H}(z_\star)&\geq \mathcal{H}(z)+\nabla\mathcal{H}(z)^\top(z_\star-z)+\tfrac{1}{2}(z_\star-z)^\top Q(z)(z_\star-z)\quad Q(z)\in\partial(\nabla\mathcal{H})(z)\nonumber\\
	&\geq \mathcal{H}(z)+\nabla\mathcal{H}(z)^\top(z_\star-z)+\tfrac{\upmu_\star}{2}\|z_\star-z\|_2^2
\end{align}
where in the last passage we have applied~\eqref{eq:matr_bound}. Re-ordering the terms, it holds that
\begin{align}
	\mathcal{H}(z)-\mathcal{H}(z_\star)&\leq -\nabla\mathcal{H}(z)^\top(z_\star-z)-\tfrac{\upmu_\star}{2}\|z_\star-z\|_2^2\nonumber\\
	&\leq \nabla\mathcal{H}(z)^\top(z-z_\star)-\tfrac{\upmu_\star}{2}\|z-z_\star\|_2^2
\end{align}
The r.h.s. is a function $ar-\upmu_\star r^2/2$ that reaches a maximum in $r>0$ for $r=a/\upmu_\star$. Taking now $a=\nabla\mathcal H$ and $r=z-z_\star$ we obtain
\begin{equation}
	\mathcal{H}(z)-\mathcal{H}(z_\star)\leq\tfrac{1}{2\upmu_\star}\|\nabla\mathcal{H}(z)\|_2^2.
\end{equation}
\end{proof}

For a one-hidden-layer hybrid Hamiltonian,
\begin{equation}
    \mathcal H(z)=\frac12\|z\|_2^2-\mathcal F_2(W_{21}z),
    \qquad
    \nabla\mathcal H(z)=z-W_{21}^\top\Psi_2(W_{21}z).
    \label{eq:one_layer_gradient}
\end{equation}
If $\Psi_2$ is locally Lipschitz, the Clarke sum and chain rules give the sound inclusion
\begin{equation}
    \partial(\nabla\mathcal{H})(z_\star)
    \subseteq
    \left\{I-W_{21}^\top QW_{21}:
    Q\in\partial(\Psi_2)(W_{21}z_\star)\right\}.
    \label{eq:one_layer_clarke_hessian}
\end{equation}
Thus the tractable sufficient condition
\begin{equation}
    \inf_{Q\in\partial(\Psi_2)(W_{21}z_\star)}
    \uplambda_{\min}\!\left(I-W_{21}^\top QW_{21}\right)>0
    \label{eq:activation_clarke_margin}
\end{equation}
implies Proposition~\ref{prop:hessian_pl}. The inclusion may be strict for nonsmooth compositions, but positivity over the displayed outer set remains a sound certificate.

For the one-hidden-layer Hamiltonian~\eqref{eq:one_layer_gradient}, the gradient is the affine term $z$ minus the composition of the locally Lipschitz activation $\Psi_2$ with two linear maps. The Clarke generalized-Jacobian sum and chain rules yield~\eqref{eq:one_layer_clarke_hessian}. Therefore, requiring every matrix $I-W_{21}^\top QW_{21}$ with $Q\in\partial(\Psi_2)(W_{21}z_\star)$ to be uniformly positive definite is a directly checkable sufficient condition for Proposition~\ref{prop:hessian_pl}. The following specializations make this condition explicit.

\paragraph{Softmax.}
Let
\begin{equation}
    \mathcal F_2(s)=\frac1\beta\log\left(\sum_{i=1}^{N_2}e^{\beta s_i}\right),
    \qquad
    p(s)=\Psi_2(s)=\operatorname{softmax}(\beta s),
    \qquad \beta>0.
\end{equation}
Then $\partial(\Psi_2)(s)=\{D\Psi_2(s)\}$ with
\begin{equation}
    D\Psi_2(s)
    =\beta\left[\operatorname{diag}(p(s))-p(s)p(s)^\top\right],
\end{equation}
and a sufficient local PL condition at $z_\star$ is
\begin{equation}
    I-\beta W_{21}^\top
    \left[\operatorname{diag}(p_\star)-p_\star p_\star^\top\right]
    W_{21}\succ0,
    \qquad p_\star=p(W_{21}z_\star).
    \label{eq:softmax_pl_condition}
\end{equation}
Softmax is smooth and bounded, so it is also compatible with the bounded-first-hidden-layer condition used to establish coercivity.

\paragraph{Hyperbolic tangent.}
For the componentwise activation $\Psi_{2,i}(s_i)=\tanh(\beta s_i)$, one may take
\begin{equation}
    \mathcal F_2(s)=\sum_i\frac1\beta\log\cosh(\beta s_i),
\end{equation}
with
\begin{equation}
    \partial(\Psi_2)(s)
    =\left\{\beta\operatorname{diag}\!\left(\operatorname{sech}^2(\beta s_i)\right)\right\}.
\end{equation}
The generalized-Hessian certificate reduces to
\begin{equation}
    I-\beta W_{21}^\top
    \operatorname{diag}\!\left(\operatorname{sech}^2(\beta (W_{21}z_\star)_i)\right)
    W_{21}\succ0.
    \label{eq:tanh_pl_condition}
\end{equation}
The activation is smooth and bounded and therefore satisfies both the local regularity requirement and the first-hidden-layer boundedness requirement.

\paragraph{Logistic sigmoid.}
For $\Psi_{2,i}(s_i)=\upsigma(\beta s_i)$, with $\upsigma(r)=(1+e^{-r})^{-1}$,
\begin{equation}
    \partial(\Psi_2)(s)
    =\left\{\beta\operatorname{diag}\!\left(\upsigma(\beta s_i)(1-\upsigma(\beta s_i))\right)\right\},
\end{equation}
and local PL follows from
\begin{equation}
    I-\beta W_{21}^\top
    \operatorname{diag}\!\left(\upsigma_i^\star(1-\upsigma_i^\star)\right)
    W_{21}\succ0,
    \qquad
    \upsigma_i^\star=\upsigma(\beta (W_{21}z_\star)_i).
    \label{eq:sigmoid_pl_condition}
\end{equation}
Sigmoid is again smooth and bounded.

\paragraph{ReLU and activation kinks.}
For the componentwise ReLU activation $\Psi_{2,i}(s_i)=[s_i]_+$, a convex $C^{1,1}$ primitive is $\mathcal F_{2,i}(s_i)=\frac12[s_i]_+^2$. Its Clarke generalized Jacobian is
\begin{equation}
    \partial(\Psi_2)(s)
    =\left\{\operatorname{diag}(q):
    \begin{array}{ll}
        q_i=1,&s_i>0,\\
        q_i=0,&s_i<0,\\
        q_i\in[0,1],&s_i=0
    \end{array}
    \right\}.
    \label{eq:relu_clarke_jacobian}
\end{equation} Consequently, local PL is certified even at a kink whenever
\begin{equation}
    I-W_{21}^\top DW_{21}\succ0
    \qquad
    \text{for every }D\in\partial(\Psi_2)(W_{21}z_\star).
    \label{eq:relu_clarke_pl}
\end{equation}
Let $D_{\max}$ set the slope to one for every active or zero preactivation and to zero for every strictly inactive preactivation. Since $0\preceq D\preceq D_{\max}$ for every admissible $D$, condition~\eqref{eq:relu_clarke_pl} is equivalent to the single worst-case check
\begin{equation}
    I-W_{21}^\top D_{\max}W_{21}\succ0.
    \label{eq:relu_worst_case_pl}
\end{equation}
This removes the need to assume a fixed activation pattern around $z_\star$. ReLU remains unbounded, however, so it is not compatible with the bounded-first-hidden-layer hypothesis used above to establish global coercivity when it is chosen as $\Psi_2$; the local Clarke certificate is still applicable when ReLU appears only in deeper layers or when coercivity is guaranteed by another architectural mechanism.

\paragraph{Bounded piecewise-linear activations.}
The generalized certificate also covers nonsmooth activations that are compatible with the hybrid coercivity argument. For example, let $\Psi_{2,i}(s_i)=\operatorname{clip}(s_i,-1,1)$ (hard-$\tanh$). A convex $C^{1,1}$ primitive is
\begin{equation}
    \mathcal F_{2,i}(s_i)
    =\begin{cases}
        \frac12s_i^2,&|s_i|\leq1,\\
        |s_i|-\frac12,&|s_i|>1.
    \end{cases}
\end{equation}
The generalized slopes are one for $|s_i|<1$, zero for $|s_i|>1$, and any value in $[0,1]$ at $s_i=\pm1$. Thus the same worst-case diagonal test as~\eqref{eq:relu_worst_case_pl}, with $D_{\max}$ defined from the unsaturated and kink coordinates, certifies local PL. Unlike ReLU, hard-$\tanh$ is bounded and therefore can simultaneously satisfy the first-hidden-layer boundedness condition used for coercivity.

\paragraph{Deeper hybrid networks.}
For deeper architectures, the activation-independent statement remains Proposition~\ref{prop:hessian_pl}: it is sufficient that every element of the full projected Clarke generalized Hessian at the learned equilibrium be positive definite. Computing that set exactly may be conservative or expensive because generalized-Jacobian chain rules through the feedforward composition can yield set-valued inclusions. Any tractable outer approximation $\mathcal Q_\star$ satisfying
\begin{equation}
    \partial(\nabla\mathcal H)(z_\star)\subseteq\mathcal Q_\star
\end{equation}
therefore provides a sound sufficient certificate whenever every $M\in\mathcal Q_\star$ has a common positive spectral margin. The one-hidden-layer formulas above are the simplest exact or directly computable instances of this principle.

\section{Experimental details and reproducibility}
\label{app:experimental_details}

This section records the experimental choices used in the released code and complements the compact discussion in Section~\ref{sec:experiments}. The benchmark-specific architectures are intentionally allowed to vary: the common object across tasks is the pH-EBM construction, i.e., a learned Hopfield storage function coupled to free port-Hamiltonian interconnection, dissipation, and input maps. Unless stated otherwise, continuous-time rollouts are integrated with fourth-order Runge--Kutta (RK4), and all reported certificates are evaluated after training without modifying the learned parameters.

\begin{table}[!th]
\centering
\scriptsize
\caption{\textbf{Prediction accuracy and certified robustness across nonlinear dynamical systems.} \emph{Comparative performance of the pH-EBM across benchmarks addressing nonlinear system identification, nonconvex dynamics, and higher-dimensional controlled systems. Results reported within square brackets, $[..]$, correspond to evaluations on distinct test sets of the same benchmark, while entries separated by semicolons refer to variants of the reference architecture.}}
\label{tab:SI_experimental_results}
\resizebox{\linewidth}{!}{%
\begin{tabular}{@{}lccccc@{}}
\toprule
Benchmark & Metric $\downarrow$ & Published & portHNN-u & pH-EBM\\
\midrule
Silverbox & RMSE & $[0.289, 0.334, 0.257]$ & $[51.343, 50.981, 41.138]$ & $[0.395,0.545,0.327]$ \\
CED & RMSE & $[0.062, 0.047]$ & $[0.241, 0.366]$ & $[0.073,0.055]$ &\\
Duffing double-well & RMSE & - & $0.254$ & $0.048$ \\
$3$-link pendulum & RMSE & $0.162;0.162$ & $0.051;0.051$ & $0.014$\\
NanoDrone & RMSE & - & $[1.532,5.260,2.150,15.541]$ & $[1.368,3.550,1.343,11.648]$\\
NanoDrone (Melon) & RMSE & $[4.051,12.581,6.059,35.5735]$ & $[3.939,12.445,5.574,30.802]$ & $[11.860,28.014,7.169,36.597]$ \\
\bottomrule
\end{tabular}}
\end{table}

\subsection{Datasets, splits, and preprocessing}

\paragraph{Silverbox.}\label{par:Silverbox}
We use the canonical SISO Silverbox benchmark through the \texttt{nonlinear\_benchmarks} interface. Input and output are standardized using the mean and standard deviation of the training record only; the same affine transformation is then applied unchanged to every test record. The final model uses a four-dimensional latent state. Its initial state is inferred from a burn-in window before the free rollout begins. The released evaluation scores all three official test records---multisine, full arrow, and arrow without extrapolation - rather than selecting a single favorable trajectory. The RMSE score reported in the main is the average RMSE over the three test datasets: check Table~\ref{tab:SI_experimental_results} for the individual scores. The champion is trained on windows of 192 prediction steps with stride 16 and batch size 32. The benchmark-native output RMSE is reported in mV, together with the dimensionless NRMSE used internally for diagnostics.

\begin{figure}[!h]
\centering
\subfloat[Projected safe-set geometry and input-radius sweep.]{\includegraphics[width=.85\linewidth]{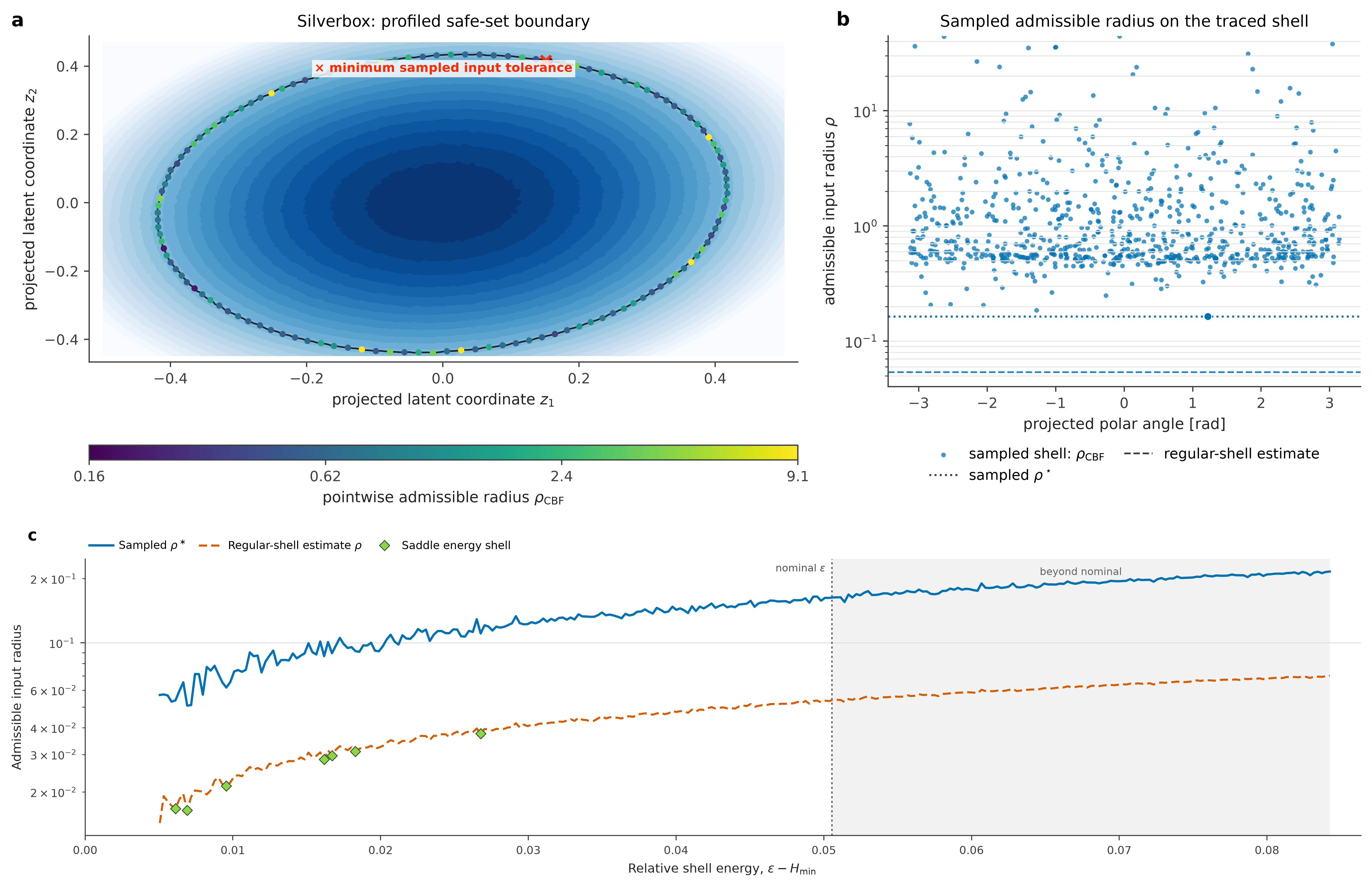}}

\subfloat[Affine Hamiltonian slice and profiled energy envelope.]{\includegraphics[width=.85\linewidth]{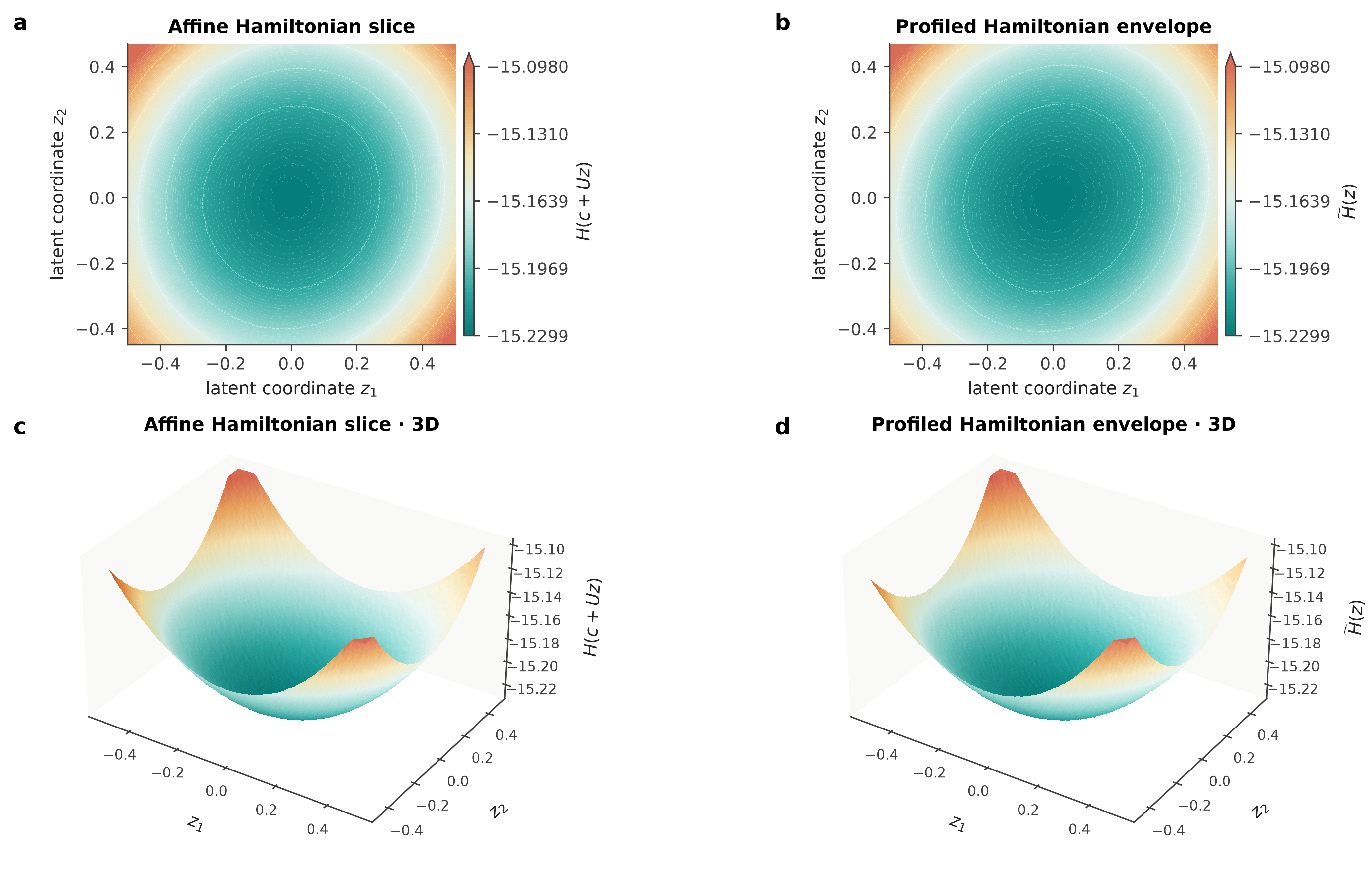}}
\caption{\textbf{Silverbox: learned energy geometry and certificate diagnostics.} The upper composite shows the profiled two-dimensional energy view at the nominal shell: traced boundary states are colored by their pointwise admissible-input radius, the red cross marks the least tolerant sampled state, and the adjacent panel resolves the same radii along the shell. The lower sweep compares the sampled shell minimum with the regular-shell estimate as the relative energy budget increases; green diamonds mark detected saddle-energy shells. The second composite contrasts an affine Hamiltonian slice with the profiled envelope in both 2D and 3D. Their close agreement indicates that the selected plane captures the dominant single-well geometry over the explored region. These plots visualize the \emph{learned} Hamiltonian and should not be interpreted as recovery of a unique physical Silverbox potential.}
\label{fig:si_silverbox}
\end{figure}

\paragraph{Coupled Electric Drive (CED).}\label{par:CED}
CED contains two SISO realizations corresponding to different actuation amplitudes. Both training realizations are concatenated and standardized with statistics computed jointly from the combined training data, while the two test records remain separate and are rolled out independently, with no state carried between them. The benchmark-provided state-initialization window has length 10. The selected configuration uses a four-dimensional latent state, 64-step training windows, unit stride, and batch size 16. We report each test RMSE in ticks/s and retain the mean NRMSE only as an aggregate diagnostic. This separation is important because the two records probe distinct operating amplitudes rather than repeated measurements of the same trajectory.

\begin{figure}[!th]
\centering
\subfloat[Projected safe-set boundary and epsilon sweep.]{\includegraphics[width=.85\linewidth]{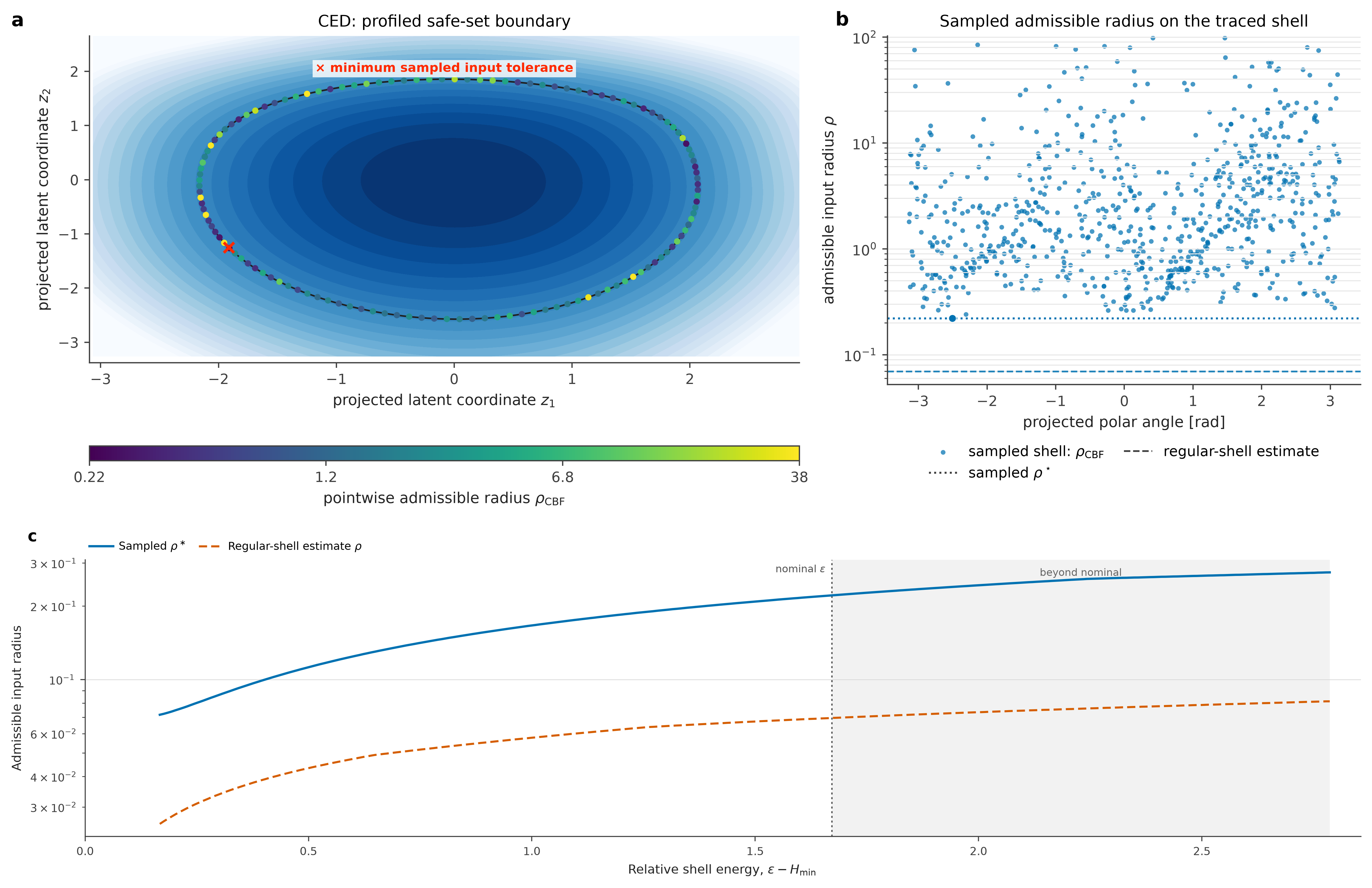}}

\subfloat[Affine slice and profiled Hamiltonian view.]{\includegraphics[width=.85\linewidth]{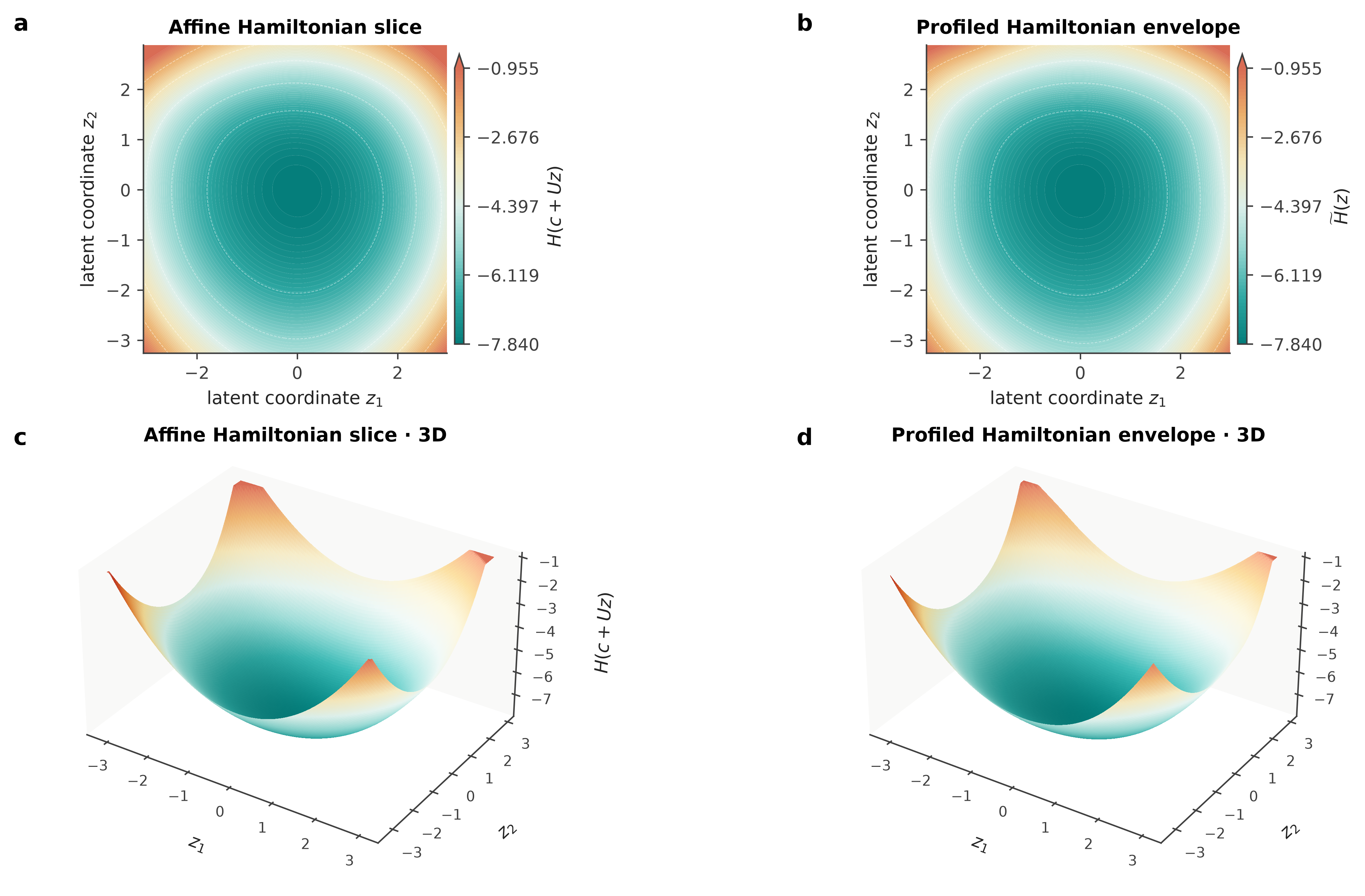}}
\caption{\textbf{CED: learned energy geometry and certificate diagnostics.} The upper composite shows the profiled safe-set boundary, colored by pointwise admissible-input radius, together with the shell-wise radius distribution and the least tolerant sampled state. The energy sweep compares the sampled minimum radius with the regular-shell estimate from the nominal shell into higher-energy regions. The second composite compares the affine slice and profiled envelope in 2D and 3D. The smooth, nearly monotone radius trend and close slice/profile agreement are consistent with an effectively single-well learned geometry over the explored region; they do not establish convexity of an underlying physical energy.}
\label{fig:si_ced}
\end{figure}

\paragraph{Duffing double well.}\label{par:duffing}
The Duffing experiment is state observed and therefore differs deliberately from the output-only identification benchmarks. Data are generated from the physical dynamics
$\dot q=p$, $\dot p=q-q^3-rp+u$ with $r=0.4$, sampling interval $\Delta t=0.02$ s, and 20 s trajectories. The frozen dataset contains 200 training, 40 validation, and 100 nominal-test trajectories. Initial conditions alternate between the two wells, with 20\% of the trajectories drawn from a higher-energy regime. Training inputs combine piecewise-constant, pseudo-random binary, and multisine forcing. Four held-out forcing families - step, chirp, unseen-frequency sinusoid, and long constant forcing - are generated only for OOD evaluation. Each OOD family contains 25 trajectories in the supplied artifact. No state or input normalization is used: the learned state is directly $z=(q,p)$, the physical initial condition is provided to the model, and the state itself is the prediction target. This makes the learned Hamiltonian directly comparable with the analytic mechanical energy.

\begin{figure}[!th]
\centering
\includegraphics[width=\linewidth]{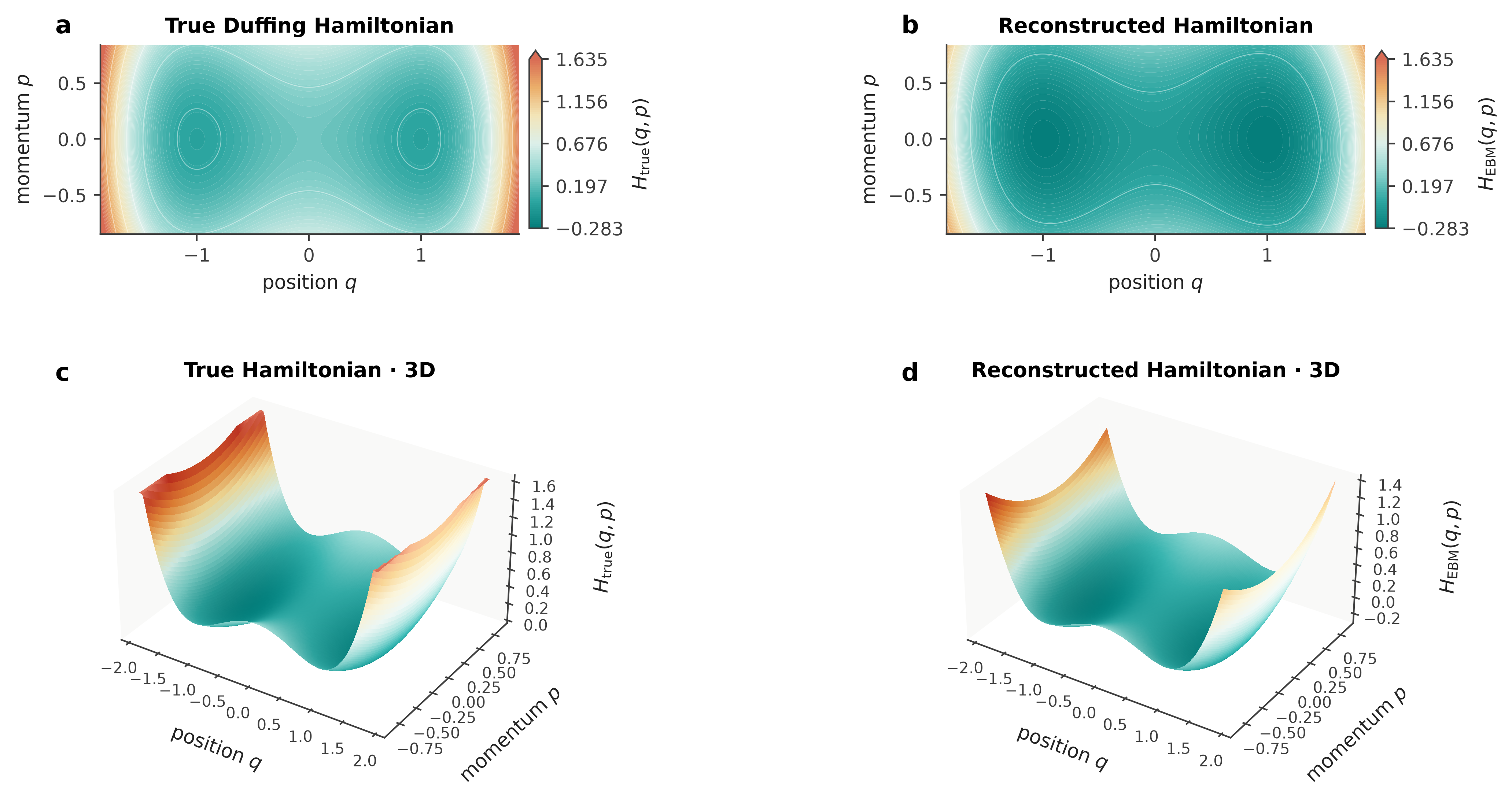}
\caption{\textbf{Duffing Hamiltonian reconstruction in physical coordinates.} Analytic (left) and learned pH-EBM (right) energies are shown directly in $(q,p)$, so no latent projection is involved. Both the contour and 3D views recover the characteristic double-well topology and the central saddle separating the two basins. The comparison supports qualitative recovery of the physical energy geometry; it does not imply unique or exact identification of the underlying Hamiltonian.}
\label{fig:si_duffing_hamiltonian}
\end{figure}

\paragraph{$n$-link pendulum.}\label{par:nlink}
The pendulum experiments import the arrays produced by the reference Deep Dissipative Dynamics data generator; the pH-EBM code does not regenerate an ``equivalent'' dataset. Inputs, partial observations, and full states are loaded from the corresponding \texttt{.npy} files. The primary mode is output-only: the observation consists of the first-joint angle/velocity pair, whereas the complete physical state is retained only for diagnostics. All supplied trajectories use $\Delta t=0.05$ s and duration 10 s. The pH-EBM uses an eight-dimensional latent state, an input-aware 10-sample encoder, and the same output-only convention for both the two- and three-link experiments. The final released configurations use 50-step windows, stride 50, batch size 8, and 500 epochs. The supplementary figures report both nominal reconstruction/OOD diagnostics and the learned projected energy geometry. The supplied release contains complete champion artifacts for $n=2$ and $n=3$.

\begin{figure}[!th]
\centering
\subfloat[$n=2$: prediction, OOD tests, and certificate diagnostics.]{\includegraphics[width=\linewidth]{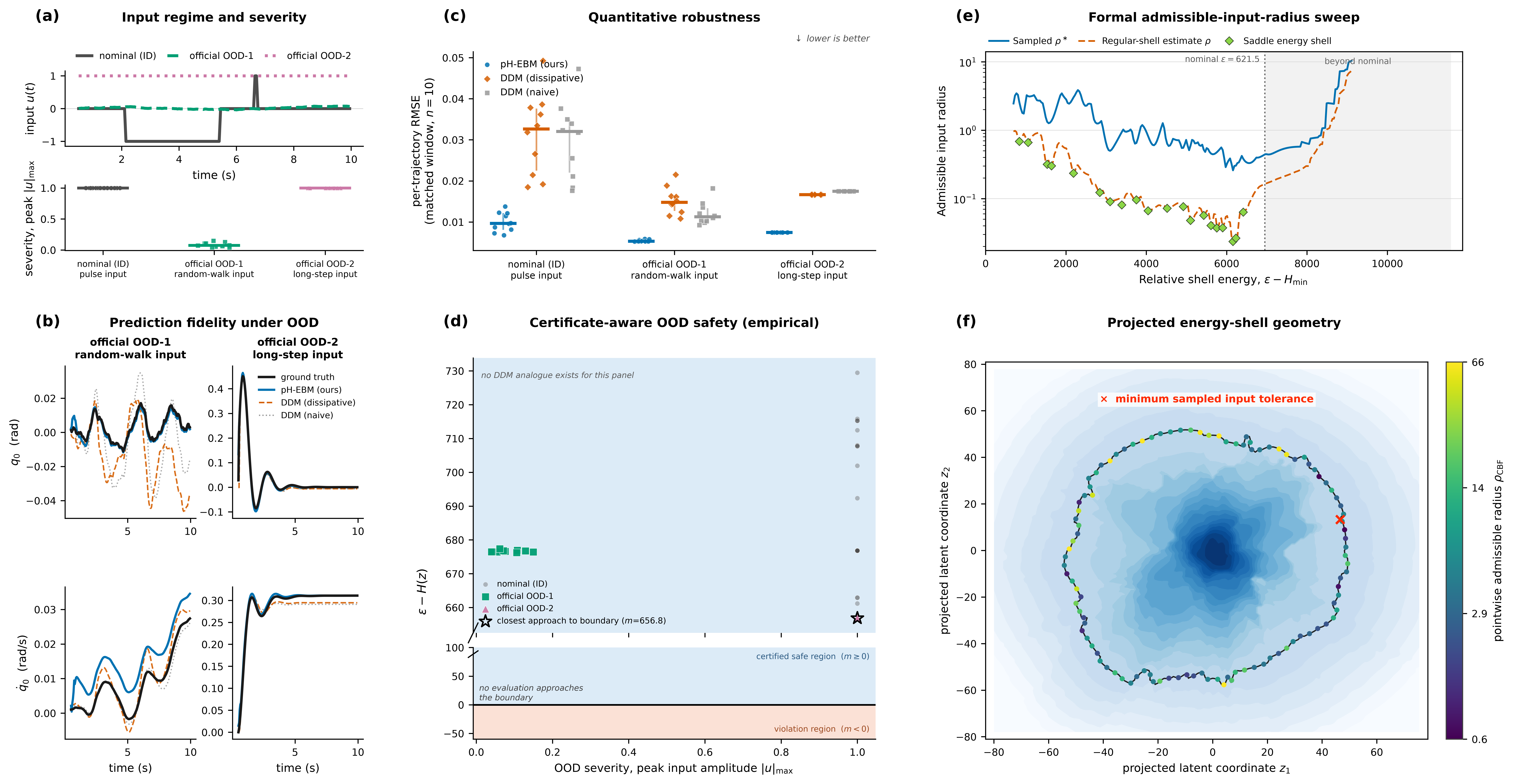}}

\subfloat[$n=2$: affine slice versus profiled envelope.]{\includegraphics[width=.95\linewidth]{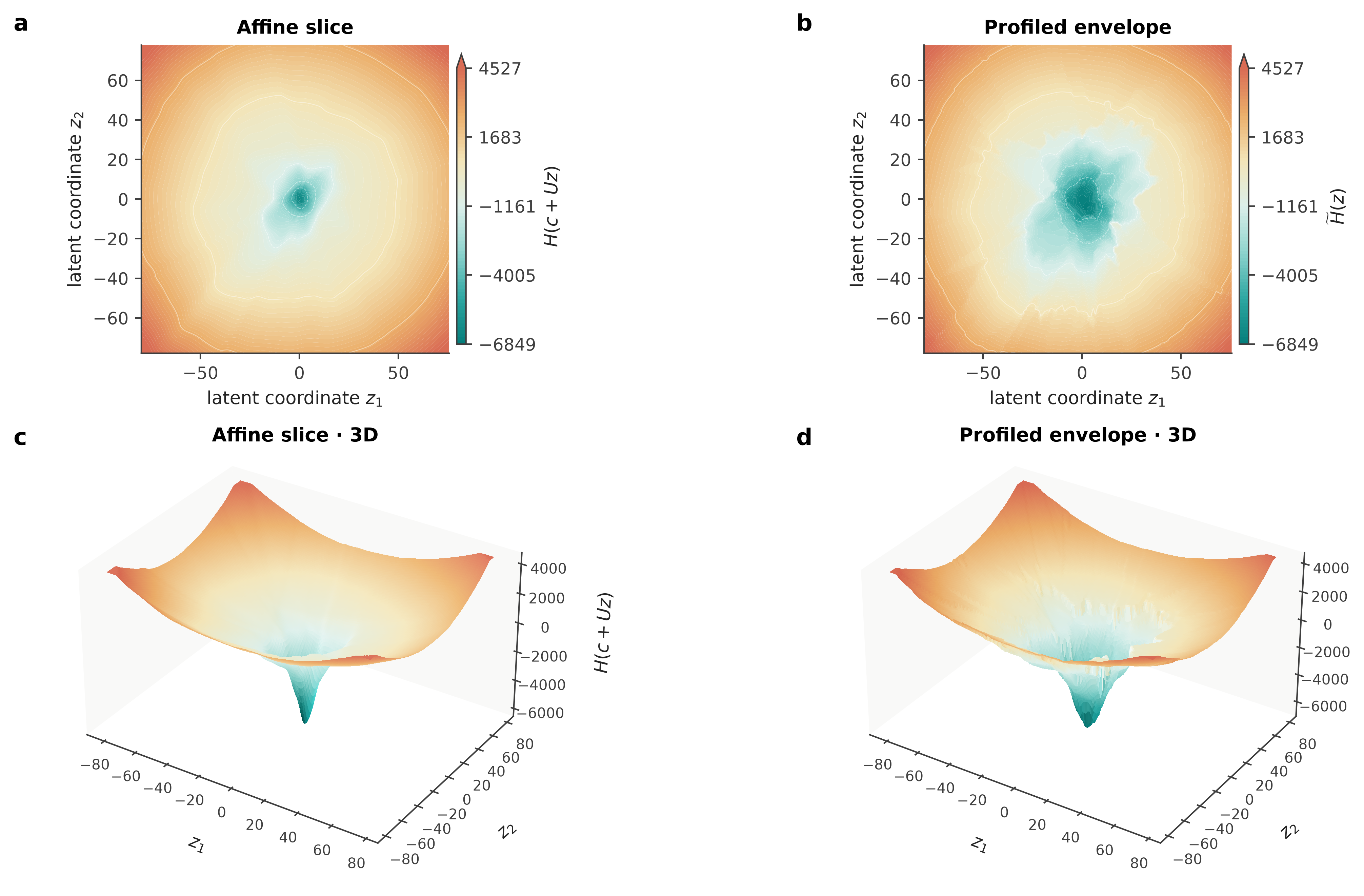}}
\caption{\textbf{$2$-link pendulum: prediction, OOD forcing, and certificate geometry.} The first composite reports the nominal pulse input and the two official OOD inputs, representative trajectory reconstructions, per-trajectory RMSE across pH-EBM and dissipative/naive DDM baselines, the empirical safety margin of the evaluated OOD trajectories, the admissible-input-radius sweep across energy shells, and the projected shell colored by pointwise radius. The radius decreases near several detected saddle-energy shells before increasing beyond the nominal operating energy. The second composite compares an affine Hamiltonian slice with the profiled envelope, showing the nonconvex structure that is hidden by any single trajectory projection.}
\label{fig:si_nlink_n2}
\end{figure}

\begin{figure}[!th]
\centering
\subfloat[$n=3$: prediction, OOD tests, and certificate diagnostics.]{\includegraphics[width=\linewidth]{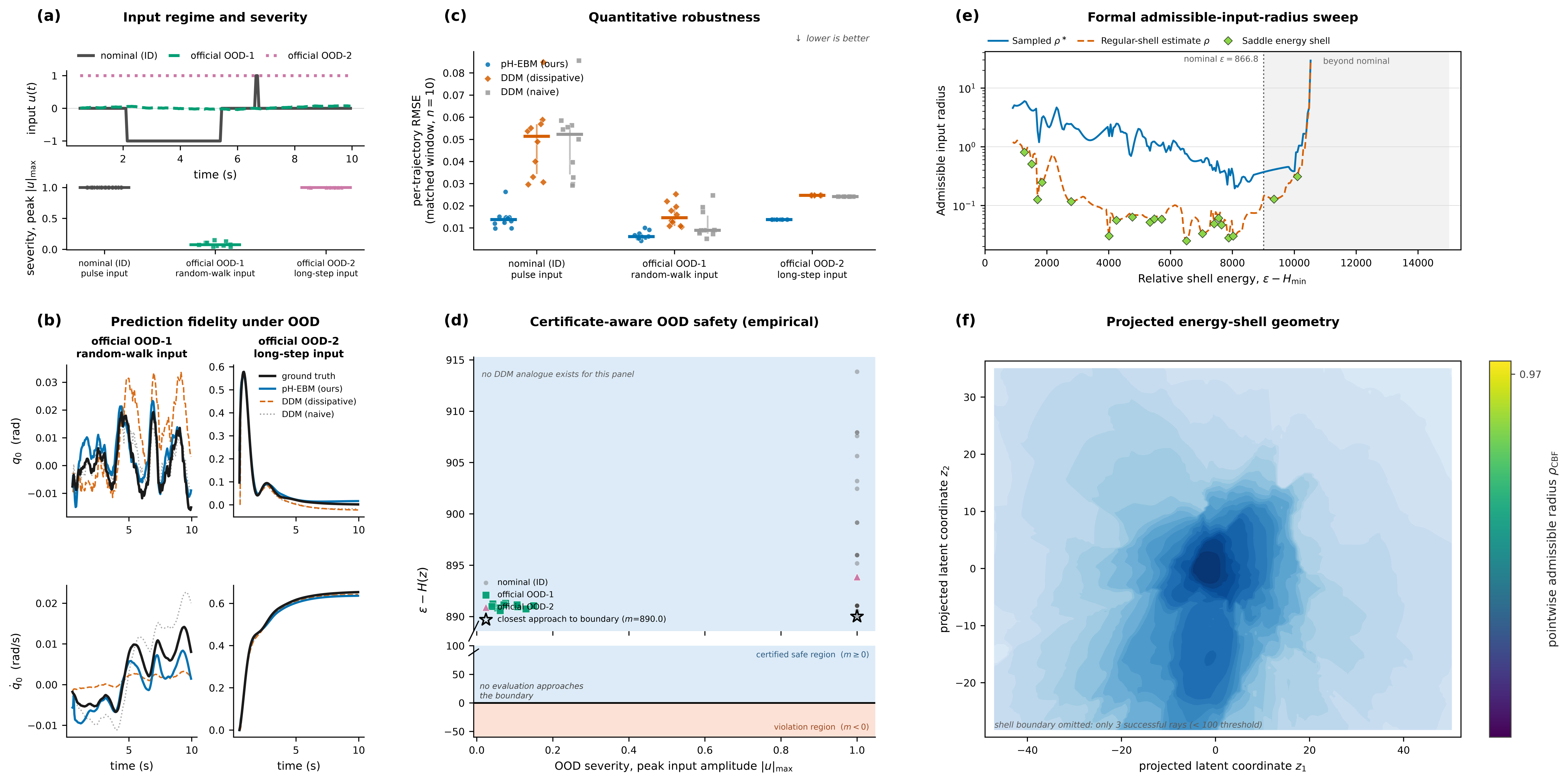}}\\[-1mm]

\subfloat[$n=3$: affine slice versus profiled envelope.]{\includegraphics[width=.95\linewidth]{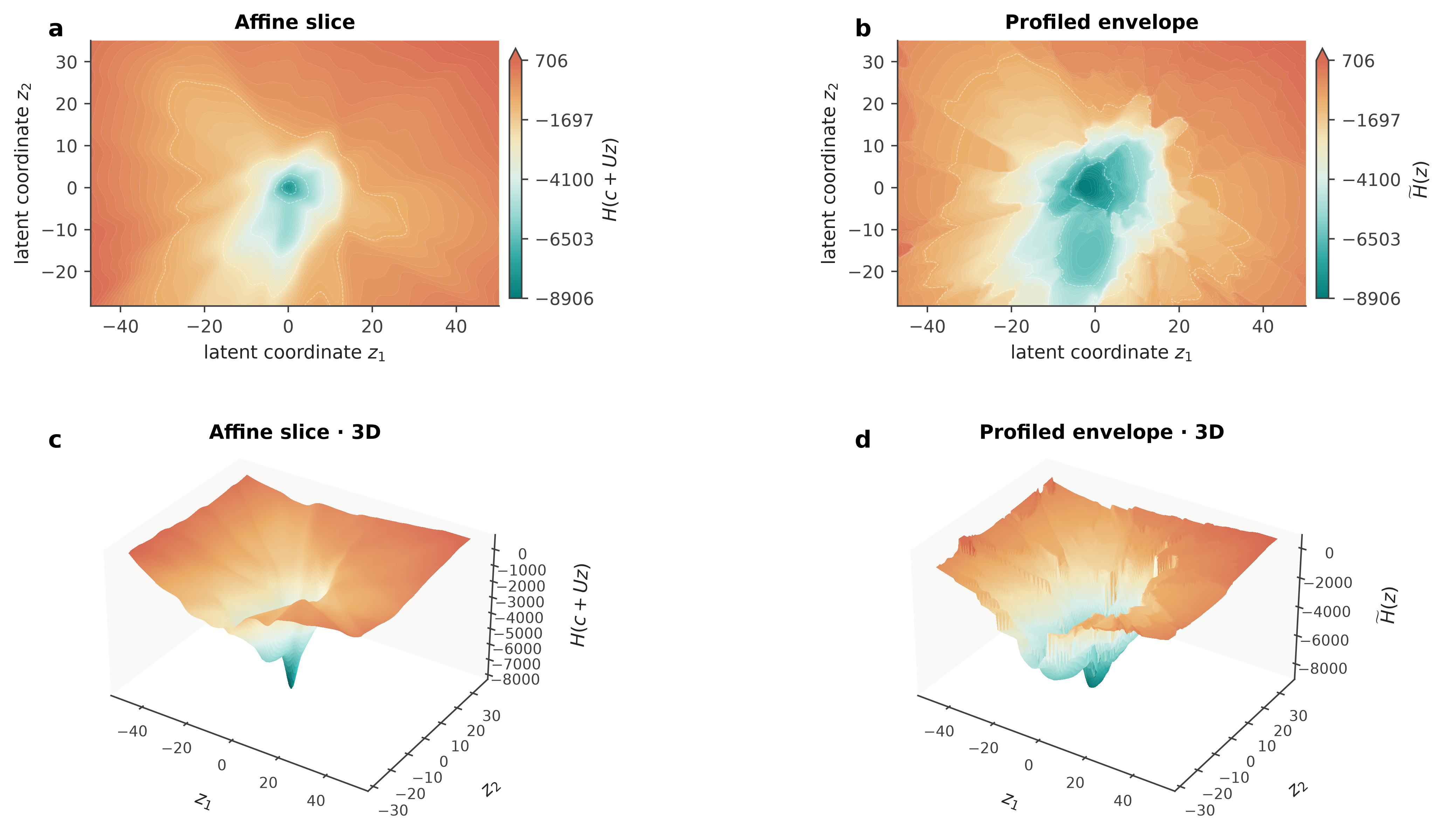}}
\caption{\textbf{$3$-link pendulum: prediction, OOD forcing, and certificate geometry.} The first composite mirrors the $2$-link protocol: nominal and official OOD inputs, representative reconstructions, per-trajectory RMSE, empirical safety margins, and an energy-shell sweep of the sampled and regular-shell admissible radii. Detected saddle-energy shells coincide with pronounced changes in the radius curve. The projected energy field is shown for geometric context, but the shell boundary is intentionally omitted because too few rays met the reconstruction threshold; it should therefore not be read as a complete projected boundary certificate. The second composite compares affine-slice and profiled-envelope views of the learned nonconvex Hamiltonian.}
\label{fig:si_nlink_n3}
\end{figure}

\paragraph{NanoDrone.}
The NanoDrone loader follows a frozen external dataset revision and validates all file names, timestamps, and quaternion norms before training. The training families are Square, Random, and Chirp, with four runs per family; runs 1-3 form the development training split and run 4 is used for development validation. The three Melon trajectories form the held-out extrapolation family and are explicitly excluded from scaler fitting and from the S3 model-selection protocol. Samples are spaced by $\Delta t=0.01$ s. The physical state is represented by 12 coordinates: position, linear velocity, the $SO(3)$ logarithm of attitude, and angular velocity. The four inputs are rotor angular speeds. State and input scalers are fitted only on the training trajectories, and the primary pH-EBM is trained in state-observed mode with an identity readout. The final S3 confirmation campaign evaluates horizons of 50, 100, 250, 500, and 1000 steps; the primary long-horizon comparison is 500 steps (5 s), ten times the 50-step (0.5 s) training/evaluation horizon emphasized by the benchmark. Four frozen seeds (11, 23, 47, 89) are used in the paired pH/black-box confirmation protocol, with model selection performed on the training-supported families only.

\begin{figure}[!th]
\centering
\subfloat[Chirp long-horizon prediction and learned safe envelope.]{\includegraphics[width=.85\linewidth]{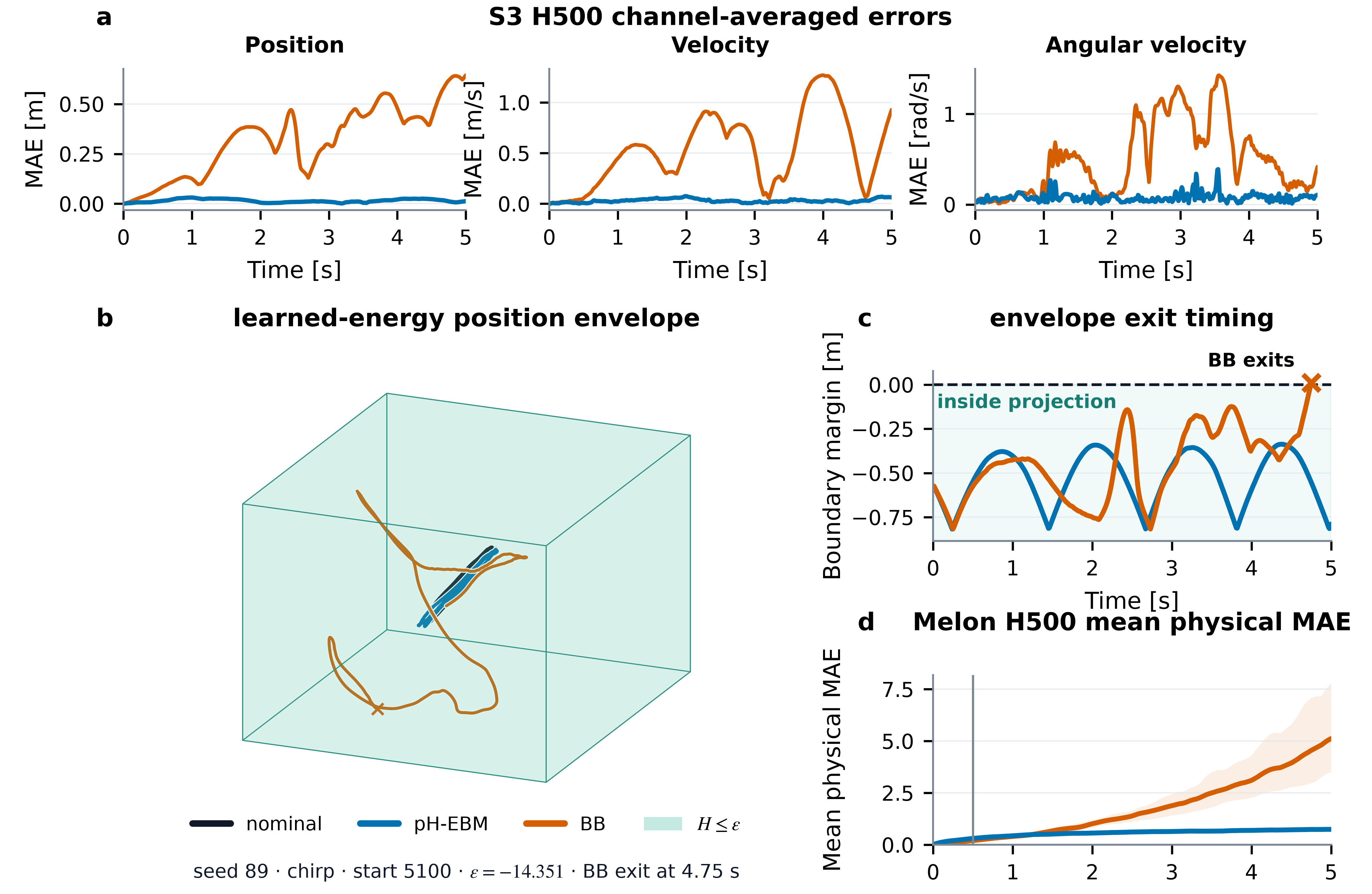}}

\subfloat[Square long-horizon prediction and learned safe envelope.]{\includegraphics[width=.85\linewidth]{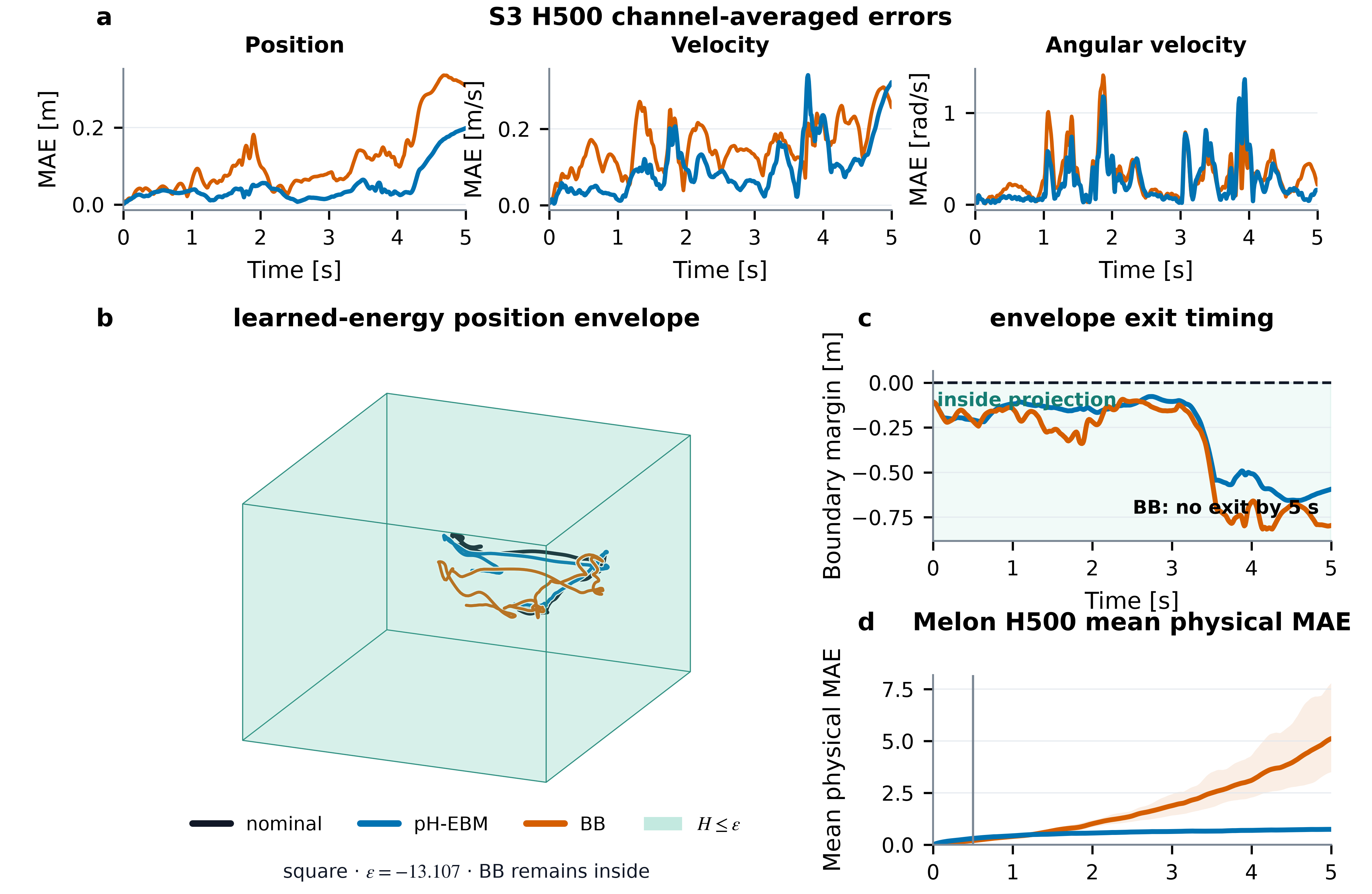}}\\[-1mm]
\caption{\textbf{NanoDrone: training-supported long-horizon Chirp and Square rollouts.} For each excitation, the top panels report channel-averaged position, velocity, and angular-velocity errors; the 3D panel compares the nominal trajectory with pH-EBM and black-box rollouts inside the projected learned-energy envelope; the exit-timing panel tracks the most critical projected coordinate; and the final panel reports physical MAE over the extended horizon. The pH-EBM remains inside the displayed envelope for the shown $5\,\mathrm{s}$ rollouts, whereas the black-box trajectory exits at the indicated time. The envelope is a visualization of the high-dimensional certificate, not a substitute for its full-state evaluation.}
\label{fig:si_nanodrone_rollouts}
\end{figure}

\begin{figure}[!th]
\centering

\subfloat[Held-out Melon long-horizon prediction.]{\includegraphics[width=.85\linewidth]{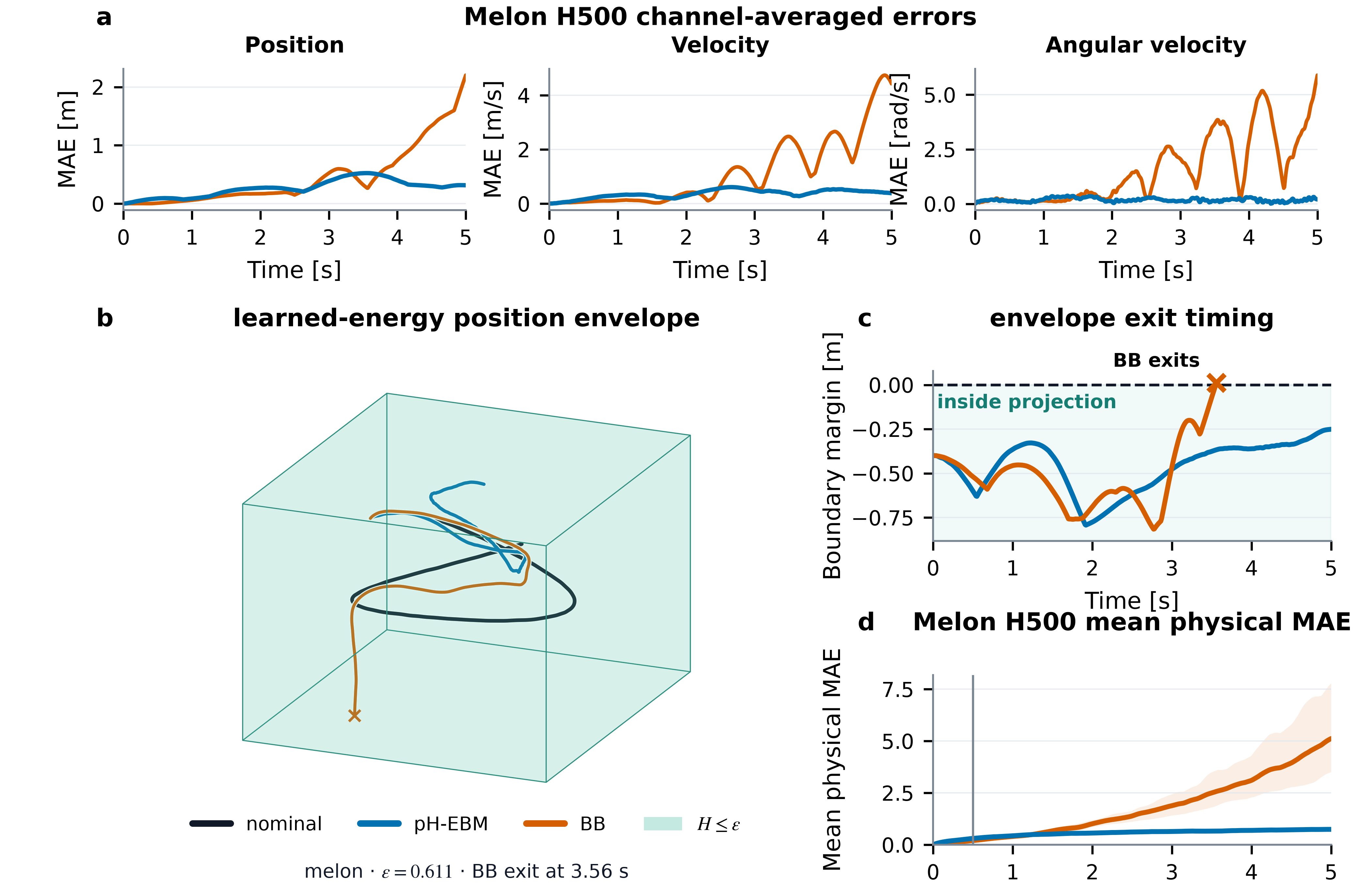}}

\subfloat[Projected Hamiltonian and profiled envelope.]{\includegraphics[width=.85\linewidth]{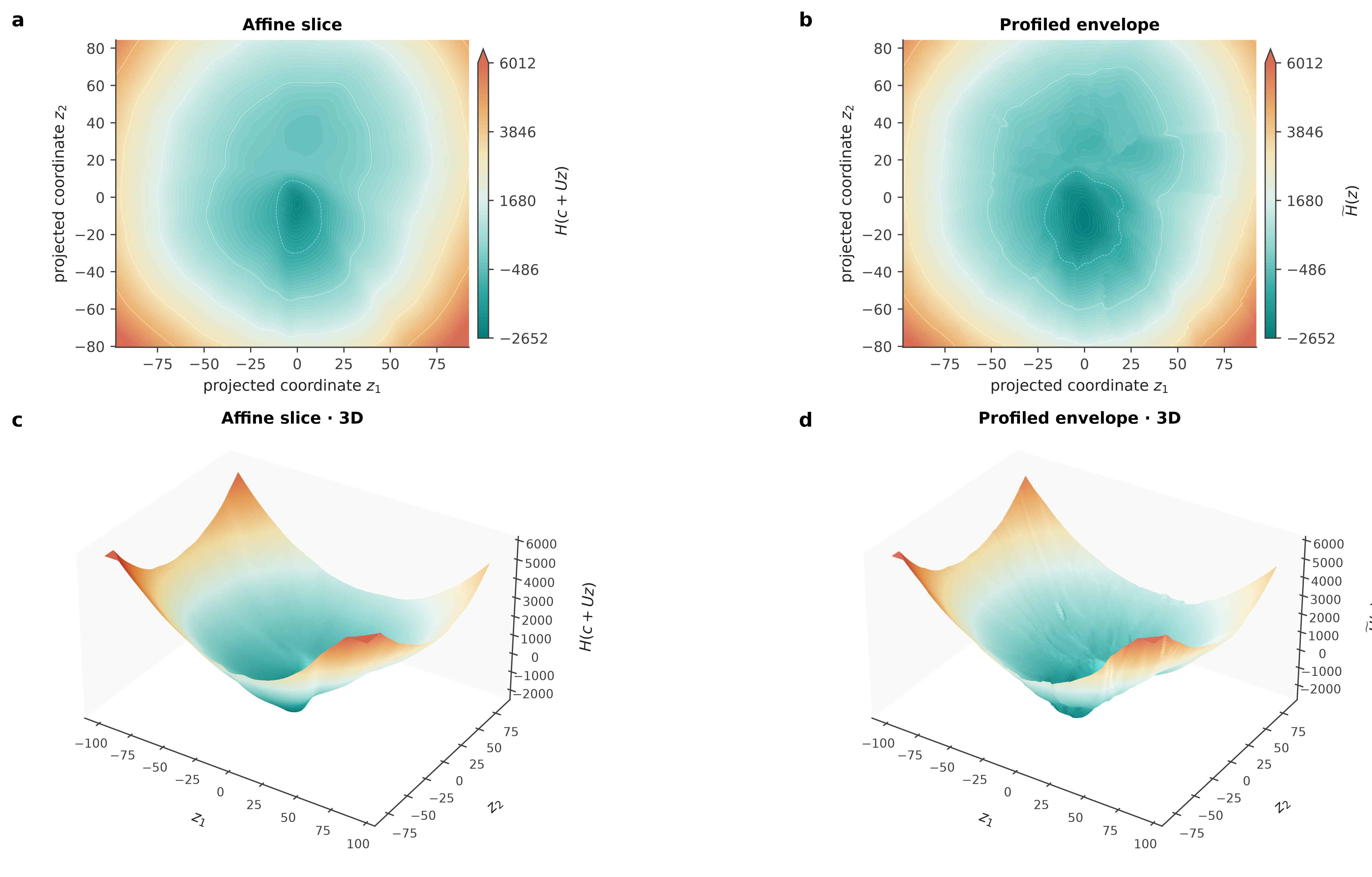}}
\caption{\textbf{NanoDrone: held-out Melon extrapolation and learned Hamiltonian geometry.} The first composite evaluates the unseen Melon family over a $5\,\mathrm{s}$ rollout, reporting channel-wise error, projected trajectory evolution relative to the learned-energy envelope, exit timing, and physical MAE. The black-box prediction is initially competitive but develops rapidly growing long-horizon error and leaves the displayed envelope, whereas the pH-EBM remains bounded in the shown rollout. The second composite compares an affine two-dimensional Hamiltonian slice with the profiled envelope obtained by optimizing over unshown coordinates, in both contour and 3D views, exposing the nonconvex geometry of the learned high-dimensional energy.}
\label{fig:si_nanodrone_melon_energy}
\end{figure}

\begin{figure}[!t]
\centering
\includegraphics[width=\linewidth]{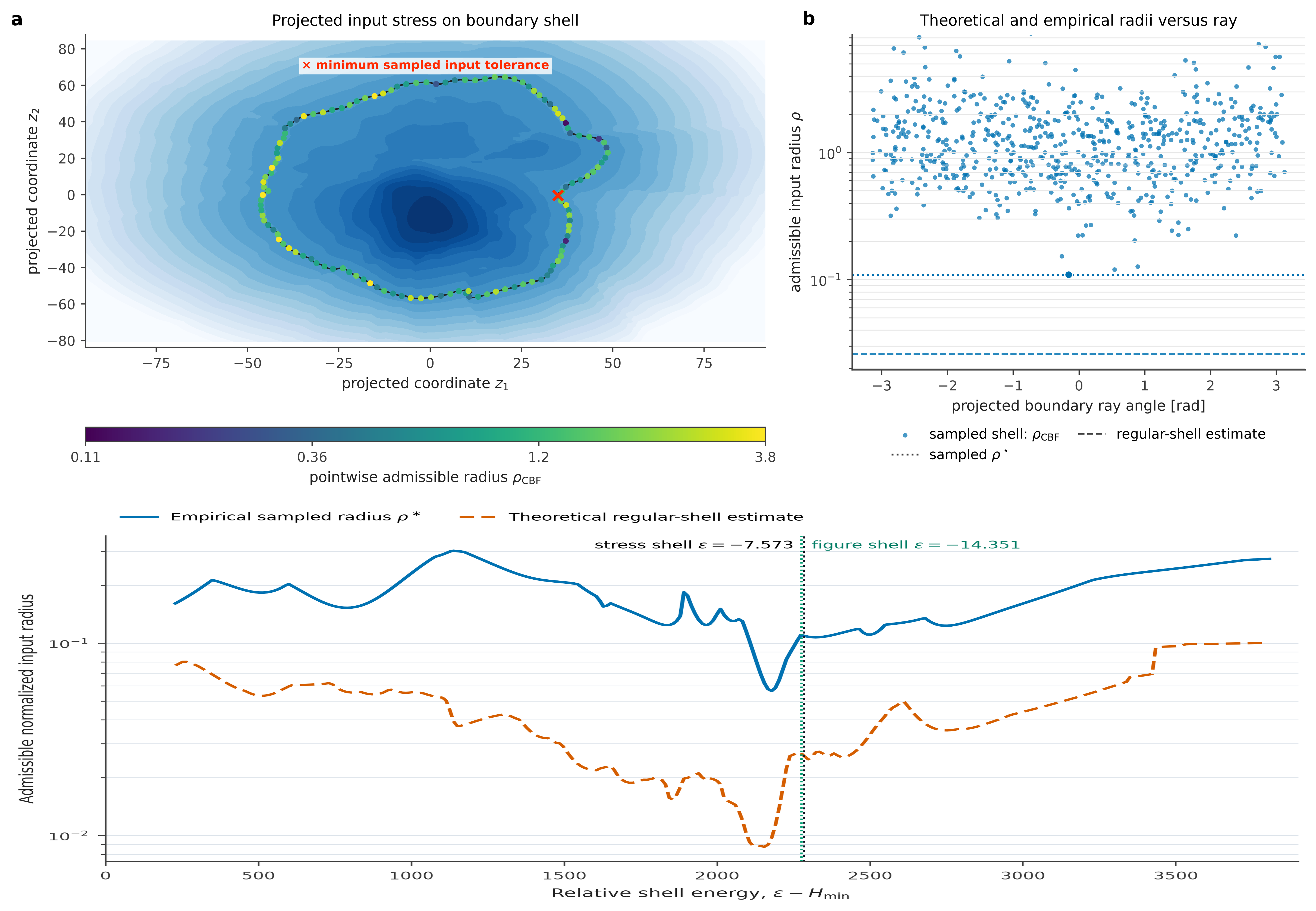}
\caption{\textbf{NanoDrone projected certificate.} Left: projection of the traced full-state energy shell, with boundary samples colored by their pointwise admissible-input radius and the least tolerant sampled state marked explicitly. Right: the same pointwise radii ordered along the projected shell, together with the sampled shell minimum and regular-shell estimate. The projection is diagnostic only: certification is evaluated in the full model state space, and the displayed contour is not asserted to contain every high-dimensional boundary extremum.}
\label{fig:si_nanodrone_certificate}
\end{figure}

\subsection{Per-benchmark model configurations}

The selected models share the same pH-EBM construction but not a forced common width. This is intentional: state dimension, observation structure, and input dimension differ substantially across the tasks. In all code-supported configurations below, $J$ is free and skew-symmetric by construction, $R$ is learned through a Cholesky factor and is positive semidefinite (strictly positive definite when the learned factor has nonzero diagonal), and $G$ is learned rather than tied to the scalar Hamiltonian parameterization. ``Active parameters'' count only terms used by the selected computational graph.

\begin{table}[!t]
\centering
\scriptsize
\caption{\textbf{Per-benchmark pH-EBM configurations.} Architectures and active parameter counts reconstructed from the supplied champion configurations and frozen provenance artifacts. Entries marked ``not in release'' correspond to rows already present in the draft for which no experiment code or champion configuration was supplied.}
\label{tab:architectures}
\resizebox{\linewidth}{!}{%
\begin{tabular}{@{}lccccccc@{}}
\toprule
Benchmark & $d_z$ & Hamiltonian architecture & Activation & $J$ & $R$ & $G$ & Parameters \\
\midrule
Silverbox & 4 & $128\!\rightarrow\!64$ & softmax($2$)$\rightarrow$poly($4$) & free skew & $LL^\top$ & free $G\in\mathbb R^{4\times1}$ & $\approx9.1$k \\
CED & 4 & $96\!\rightarrow\!48$ & softmax($2$)$\rightarrow$poly($4$) & free skew & $LL^\top$ & free $G\in\mathbb R^{4\times1}$ & $\approx5.25$k \\
Duffing double-well & 2 & $64\!\rightarrow\!32$ & $\tanh(2)\rightarrow$poly($4$) & free skew & $LL^\top$ & free $G\in\mathbb R^{2\times1}$ & $\approx2.28$k \\
$n$-link pendulum & 8 & $64\!\rightarrow\!32$ & $\tanh(2)\rightarrow$poly($4$) & free skew & $LL^\top$ & free $G\in\mathbb R^{8\times1}$ & $\approx2.99$k \\
NanoDrone$^{\dagger}$ & 12 & $128\!\rightarrow\!64$ & $\tanh\rightarrow$poly($4$) & free pH field & learned PSD & unrestricted 4-port & 43,738 \\
\bottomrule
\end{tabular}}
\end{table}

\subsection{Training, validation, and model selection}

Across the standard pH-EBM interfaces, optimization is performed on finite rollout windows with a Huber observation loss, Adam/AdamW updates, global gradient clipping, and RK4 integration. The first training epochs use a stop-gradient rollout warm-up to limit long-horizon gradient pathologies; subsequent epochs differentiate through the complete rollout. This warm-up changes the optimization path only and does not alter the model evaluated at test time.

\paragraph{Benchmark-specific schedules.}
The Silverbox champion uses 800 epochs, learning rate $4.59\times10^{-4}$, 192-step windows, stride 16, and batch size 32. CED uses a longer 8000-epoch schedule, learning rate $10^{-4}$, 64-step windows, stride 1, and batch size 16. Duffing uses 1000 epochs, learning rate $1.5\times10^{-3}$, 100-step windows, stride 25, and batch size 64 in the frozen champion. The two supplied $n$-link champions use 500 epochs, learning rate $3\times10^{-4}$, 50-step windows, stride 50, and batch size 8. For NanoDrone, the final S3 confirmation protocol fine-tunes a frozen parent model for 120 epochs with fresh Adam, batch size 64, a five-epoch warm-up, and a learning-rate schedule peaking at $3\times10^{-5}$ and ending at $10^{-7}$. Validation is evaluated at a fixed list of epochs, and the S3 campaign is repeated for four seeds. The Melon family is inaccessible to this selection stage by protocol.

\paragraph{Initial conditions and readouts.}
Silverbox, CED, and the primary $n$-link experiments are output-only systems. Their latent initial condition is inferred from a short observation window; CED and $n$-link use an input-aware encoder so the burn-in actuation is available when estimating $z(0)$. Their measured output is distinct from the power-conjugate port output. Duffing and NanoDrone instead use observed-state training: the physical state initializes the rollout directly and no learned state decoder is required. This distinction is preserved deliberately because the corresponding datasets expose different information to the learner.

\subsection{Baselines and comparison protocol}

The baseline is chosen to answer the scientific question posed by each experiment rather than forcing one architecture across unrelated datasets. Silverbox and CED are compared with benchmark-native literature references under the same published test records and physical error units. Duffing uses the released portHNN-u implementation, including a version with free full-matrix $J,R,G$, so that the comparison isolates the effect of the Hamiltonian parameterization rather than fixing the pH operators to a scalar ansatz. The $n$-link comparison imports the exact upstream Deep Dissipative Dynamics trajectories and uses its reference model on those same arrays. NanoDrone uses the accompanying black-box neural model and a paired confirmation protocol: pH and black-box models see the same S3 families, horizons, validation starts, and seed set. The frozen protocol records 43,738 active parameters for the pH model and 18,508 for the reference black-box model; the comparison is therefore a benchmark-reference comparison rather than a parameter-matched capacity ablation. Missing certification for an unconstrained baseline is denoted as unavailable, not as a zero robustness radius.

\subsection{Prediction metrics and statistical reporting}

For scalar-output identification benchmarks we report the benchmark-native root-mean-square error
\begin{equation}
    \mathrm{RMSE}=\sqrt{\frac1N\sum_{i=1}^{N}(\hat y_i-y_i)^2},
\end{equation}
and, for internal cross-dataset diagnostics, $\mathrm{NRMSE}=\mathrm{RMSE}/\operatorname{std}(y)$. Silverbox RMSE is converted back to mV and CED is reported in ticks/s. Duffing and the state-observed diagnostics use the same definitions directly in physical state coordinates. NanoDrone is evaluated in physical units after undoing the training scalers. Its primary paired S3 estimand is the difference between pH-EBM and black-box physical trajectory MAE at 500 steps, with additional reporting at 50, 100, 250, and 1000 steps and channel-family diagnostics for position, velocity, orientation, and angular velocity. Quaternion observations are converted to $SO(3)$ logarithmic coordinates for the 12-state model; orientation error is additionally tracked geometrically in the evaluation code. For consistency with official validation script, Nanodrone errors are also computed as per-output channel RMSE.

\subsection{Operational energy-shell selection}

The certificate calculations are performed after fitting. We first search the learned Hamiltonian for low-energy wells by multi-start gradient descent. Half of the starts are drawn from states visited by the fitted model when such states are available, and the remaining starts are randomized around their empirical center and scale. Descents are restricted to the operational state ball used by the numerical model, and terminal points are clustered into candidate wells. Candidates with a large energy-gradient norm are rejected before they are used as stationary anchors. This avoids treating a low-energy but nonstationary optimizer endpoint as a physical attractor.

For a chosen threshold $\epsilon$, the full-state shell $\mathcal H(z)=\epsilon$ is traced from every discovered well whose energy lies below the threshold. Random unit directions are launched from the well; along each ray the algorithm expands an outer bracket until the first crossing of the target energy is found and then refines the crossing by bisection. A short damped correction in the local gradient-normal direction finally reduces the residual $|\mathcal H(z)-\epsilon|$. Rays that do not reach the shell within the configured search radius are recorded as failures rather than silently projected to an unrelated point. Consequently, all extrema reported from this procedure are sampled shell estimates; a finite ray set by itself is not presented as an exhaustive proof of a continuous high-dimensional boundary.

\subsection{Numerical computation and verification of certificates}\label{app:num_det}

\paragraph{Hamiltonian projection.}
A two-dimensional picture of a high-dimensional Hamiltonian is necessarily a reduction, so we use two complementary views and label them separately. Let $c\in\real^d$ be the selected well, let $U\in\real^{d\times2}$ contain two orthonormal display directions, and let $V$ span the orthogonal complement. The \emph{affine slice}
\begin{equation}
    H_{\mathrm{slice}}(\zeta)=\mathcal H(c+U\zeta)
\end{equation}
is an exact cross-section: all unshown coordinates are frozen. The \emph{profiled envelope}
\begin{equation}
    \widetilde H(\zeta)=\min_{\omega}\mathcal H(c+U\zeta+V\omega)
\end{equation}
instead minimizes over the unshown coordinates. Up to numerical optimization error, the set $\{\zeta:\widetilde H(\zeta)\leq\epsilon\}$ is the planar projection of the full sublevel set $\{z:\mathcal H(z)\leq\epsilon\}$. The minimization is solved independently over the display grid with Adam and multiple warm starts, including the orthogonal coordinates of discovered wells. We retain the orthogonal-gradient norm as a stationarity diagnostic for the profile.

The display plane is chosen from physically or dynamically informative directions. With multiple discovered wells, the first direction follows the inter-well geometry and the second is taken from the dominant orthogonal variation of the sampled boundary/trajectory cloud. With a single well, the default is PCA of the sampled safety boundary; ordinary trajectory PCA is used only when boundary information is unavailable. Duffing is already two dimensional and therefore requires no projection. The plane limits are derived from robust quantiles of the projected reference states, boundary samples, and discovered wells, with padding added only for readability. The supplementary slice/profile pairs in Figs.~\ref{fig:si_silverbox}, \ref{fig:si_ced}, \ref{fig:si_nlink_n2}--\ref{fig:si_nlink_n3}, and \ref{fig:si_nanodrone_rollouts}, \ref{fig:si_nanodrone_melon_energy} use this procedure.

\paragraph{Input-radius computation.}
At every sampled shell state we evaluate the raw Hamiltonian gradient and the learned $R$ and $G$ matrices, then compute the pointwise boundary margin described in Section~\ref{sec:rob_inv}. For a chosen input norm, the sampled ``exact'' shell radius is the minimum of the pointwise ratio
\begin{equation}
    \frac{d_H(z)-\sigma_{\mathcal W(z)}(\nabla\mathcal H(z))}
         {\|G(z)^\top\nabla\mathcal H(z)\|_*}
\end{equation}
over the traced states for which the input couples to the energy gradient. In parallel, we compute the regular-boundary lower estimate from the sampled extrema of $\|\nabla H\|$, normal dissipation, normal input gain, and disturbance support. The implementation records both values because they answer different questions: the first measures the least favorable sampled direction directly, while the second exposes the geometric factors entering the analytic lower bound. Both remain numerical shell estimates unless the continuous extrema are independently bounded.

\paragraph{Epsilon sweeps.}
The epsilon sweeps in the supplementary figures do not rescale a single boundary. For every energy level, the full-state shell is traced again from scratch and both radii are recomputed. By default, 300 levels are sampled in relative energy coordinates $\epsilon-H_{\min}$, from $0.10$ to $5/3$ times the nominal energy gap $\epsilon_{\mathrm{nom}}-H_{\min}$. Thus the sweep covers the certified operating shell and extends by two thirds of its nominal gap to reveal what happens near and beyond critical levels. The boundary search radius is enlarged by a factor 1.5 during this sweep so that outer shells are not artificially truncated. For Duffing, the range crosses the saddle energy; after components merge, the shell is treated globally rather than assigning points to their original well. Failed levels are stored as missing values and are not interpolated. Every sweep artifact stores the absolute $\epsilon$, $H_{\min}$, relative energy, sampled radius, regular-boundary estimate, number of successful boundary points, failed-ray fraction, and component summaries so the plotted curves can be audited independently.

\paragraph{Hamiltonian-normalized empirical certificate.}
The robustness radius $\uprho_{\epsilon,\star}$ introduced above is an exact certificate defined at a specific energy level $\epsilon$ and expressed in the native coordinates of the model input. As a consequence, its numerical value may vary under simple rescalings of either the energy function or the input variables, making direct comparisons across architectures difficult. To obtain a dimensionless quantity suitable for cross-model evaluation, we introduce an empirical normalization procedure based exclusively on validation data. 

\emph{Common energy level.} 
Given a fixed validation rollout 
\[ \{(z_i,u_i)\}_{i=1}^{N_{\mathrm{val}}}, \] 
we define the reference energy level as the $q$-quantile of the validation energy distribution, 
\begin{equation} \epsilon_q := Q_q\!\left( \{\mathcal H(z_i)\}_{i=1}^{N_{\mathrm{val}}} \right), \qquad q=0.99. 
\end{equation} 
This choice ensures that all models are evaluated on a shell corresponding to the same validation-state occupancy rather than at an arbitrary absolute energy value. 

\emph{Input normalization.} To remove the dependence on the physical units and scaling of the inputs, we normalize the input space using the empirical second moment 
\begin{equation} M_u := \frac{1}{N_{\mathrm{val}}} \sum_{i=1}^{N_{\mathrm{val}}} u_i u_i^\top = L_uL_u^\top, 
\end{equation} 
and introduce normalized coordinates \[ u=L_uv. \] 
When $M_u$ is rank-deficient, $L_u$ is restricted to the subspace effectively excited by the validation data. In these coordinates, distances are measured relative to the variability observed during validation rather than in the original input units.
 
\emph{Normalized robustness margin.} Using the runtime vector field $f_\theta^{\mathrm{run}}$, define the directional energy dissipation rate \begin{equation} \ell_H(z,u) := \nabla\mathcal H(z)^\top f_\theta^{\mathrm{run}}(z,u). 
\end{equation} 
A negative value of $\ell_H$ corresponds to local energy dissipation, whereas $\ell_H\ge 0$ identifies inputs capable of halting or reversing this decrease. We therefore define the normalized robustness margin at state $z$ as 
\begin{equation} 
r_H(z) := \inf \left\{ \|v\|_2 : \ell_H(z,L_uv)\ge 0 \right\}.
\end{equation} 
Equivalently, $r_H(z)$ is the smallest perturbation, measured in normalized input coordinates, capable of eliminating the local decrease of the energy function. By convention, $r_H(z)=0$ whenever $\ell_H(z,0)\ge0$. For input-affine dynamics, \[ f_\theta^{\mathrm{run}}(z,u) = f_\theta^{\mathrm{run}}(z,0) + B_\theta(z)u, \] the above optimization admits the closed-form expression 
\begin{equation} r_H(z) = \frac{ -\nabla\mathcal H(z)^\top f_\theta^{\mathrm{run}}(z,0) }{ \bigl\| L_u^\top B_\theta(z)^\top \nabla\mathcal H(z) \bigr\|_2 }, 
\end{equation} 
whenever the denominator is nonzero. For nonlinear forcing structures, $r_H(z)$ is obtained numerically by locating the first zero crossing of $\ell_H$ along the normalized input direction. 

\emph{Empirical normalized radius.} Let $\widehat{\Gamma}_q$ denote a numerical reconstruction of the energy shell \[ \Gamma_q := \{z:\mathcal H(z)=\epsilon_q\} \] obtained from validation data. We then define \begin{equation} 
\widehat{\rho}_{q,\mathrm{norm}} := \min_{z\in\widehat{\Gamma}_q} r_H(z), \qquad q=0.99. 
\end{equation} 
The quantity $\widehat{\uprho}_{q,\mathrm{norm}}$ measures the smallest normalized input perturbation required to destroy local energy dissipation on the validation shell. By construction, it is invariant under any strictly increasing reparameterization of $\mathcal H$ and uses a common, data-derived input geometry for all models. Since the shell itself is reconstructed numerically from finite validation trajectories, $\widehat{\uprho}_{0.99,\mathrm{norm}}$ should be interpreted as an empirical benchmarking metric for cross-model comparison. In contrast, the radius $\uprho_{\epsilon,\star}$ remains the exact continuous-state robustness certificate associated with a prescribed energy level.

\subsection{Additional benchmark results}

Figures~\ref{fig:si_silverbox}-\ref{fig:si_nanodrone_certificate} collect the supplementary outputs generated by the released experiment pipeline. In particular, the Silverbox and CED plates show that the same certificate machinery can be applied when the learned energy is effectively single-well (or almost); the Duffing figure validates recovery of a genuinely nonconvex physical energy; the $n$-link figures compare prediction/OOD behavior with high-dimensional projected geometry; and the NanoDrone figures expose the distinction between short-horizon fit and long-horizon safe recursive deployment. These visualizations are diagnostic companions to the quantitative table in the main paper and do not replace the full-state certificate computations from which the shell radii are obtained.

\subsection{Ablations and sensitivity analyses}

The released experiments prioritize ablations that alter the scientific certificate rather than exhaustive optimizer sweeps. Duffing compares the modern-Hopfield Hamiltonian with the portHNN-u reference while keeping the port-Hamiltonian interpretation explicit. The $n$-link family changes mechanical dimension under a common eight-dimensional latent pH-EBM template. NanoDrone pairs the structured model with the benchmark black-box model over identical S3 starts and horizons. Finally, the epsilon sweeps provide a post-training sensitivity analysis of the certificate itself: they reveal how the minimum shell gradient, normal dissipation, input exposure, and admissible-input radius change as the certified operating region expands toward critical energy levels.

\end{document}